\documentclass[11pt]{article}

\usepackage{fullpage,times,url}

\usepackage{amsthm,amsfonts,amsmath,amssymb,epsfig,color,float,graphicx,verbatim}
\usepackage{algorithm,algorithmic}
\usepackage{natbib}

\usepackage{hyperref}
\hypersetup{
	colorlinks   = true, %Colours links instead of ugly boxes
	urlcolor     = blue, %Colour for external hyperlinks
	linkcolor    = blue, %Colour of internal links
	citecolor   = black %Colour of citations
}

\newtheorem{theorem}{Theorem}

\newtheorem{lemma}{Lemma}
\newtheorem{corollary}{Corollary}

\newtheorem{remark}{Remark}

\newcommand{\reals}{\mathbb{R}}
\newcommand{\E}{\mathbb{E}}

\newcommand{\be}{\mathbf{e}}
\newcommand{\bx}{\mathbf{x}}
\newcommand{\bw}{\mathbf{w}}

\newcommand{\bu}{\mathbf{u}}
\newcommand{\bv}{\mathbf{v}}
\newcommand{\bz}{\mathbf{z}}

\newcommand{\Ocal}{\mathcal{O}}

\newcommand{\Gcal}{\mathcal{G}}

\newcommand{\Xcal}{\mathcal{X}}

\newcommand{\Fcal}{\mathcal{F}}
\newcommand{\Hcal}{\mathcal{H}}
\newcommand{\Rcal}{\mathcal{R}}
\newcommand{\Ncal}{\mathcal{N}}

\newcommand{\Wcal}{\mathcal{W}}

\newcommand{\norm}[1]{\|#1\|}

\newcommand{\ep}{\epsilon}

\newcommand{\p}{\mathbb{P}}

\newcommand{\secref}[1]{Sec.~\ref{#1}}
\newcommand{\subsecref}[1]{Subsection~\ref{#1}}

\renewcommand{\eqref}[1]{Eq.~(\ref{#1})}
\newcommand{\lemref}[1]{Lemma~\ref{#1}}
\newcommand{\corollaryref}[1]{Corollary~\ref{#1}}
\newcommand{\thmref}[1]{Thm.~\ref{#1}}

\title{Size-Independent Sample Complexity of Neural Networks}
\author{
	Noah Golowich\\
	Harvard University
	\and
	Alexander Rakhlin\\
	MIT
	\and
	Ohad Shamir\\
	Weizmann Institute of Science\\
	and Microsoft Research
}
\date{}

\begin{document}
	
	\maketitle
	
	\begin{abstract}
		We study the sample complexity of learning neural networks, by 
		providing 
		new bounds on their Rademacher complexity assuming norm constraints on 
		the 
		parameter matrix of each layer. Compared to previous work, these 
		complexity bounds have improved dependence on the network depth, and 
		under 
		some additional assumptions, are fully independent of the network size 
		(both depth and width). These results are derived using some novel 
		techniques, which may be of independent interest.
	\end{abstract}
	
	\section{Introduction}
	
	One of the major challenges involving neural networks is explaining their 
	ability to generalize well, even if they are very large and have the 
	potential 
	to overfit the training data 
	\citep{neyshabur2014search,zhang2016understanding}. 
	Learning theory teaches us that this must be due to some inductive bias, 
	which 
	constrains one to learn networks of specific configurations (either 
	explicitly, 
	e.g., via regularization, or implicitly, via the algorithm used to train 
	them). 
	However, understanding the nature of 
	this inductive bias is still largely an open problem. 
	
	A useful starting point is to consider the much more restricted class of 
	linear 
	predictors ($\bx\mapsto \bw^\top\bx$). For this class, we have a very good 
	understanding of how its generalization behavior is dictated by the 
	\emph{norm} 
	of $\bw$ . In 
	particular, assuming that $\norm{\bw}\leq M$ (where $\norm{\cdot}$ 
	signifies 
	Euclidean norm), and the distribution is such that
	$\norm{\bx}\leq B$ almost surely, it is well-known that the generalization 
	error (w.r.t. Lipschitz losses) given $m$ training examples scales as 
	$\Ocal(MB/\sqrt{m})$, completely independent of the dimension of $\bw$. 
	Thus,
	it is very natural to ask whether in the more general case of neural 
	networks, 
	one can obtain similar ``size-independent'' results (independent of the 
	networks' depth and width), under appropriate norm constraints on the 
	parameters. This is also a natural question, considering that the size of 
	modern neural networks
	is often much larger than the number of training examples.
	
	Classical results on the sample complexity of neural networks do not 
	satisfy this desideratum, and have a strong explicit dependence on the 
	network size. For example, bounds relying on the VC dimension (see 
	\citet{anthony2009neural}) strongly depend on both the depth and the width 
	of the network, and can be trivial once the number of parameters exceeds 
	the number of training examples. Scale-sensitive bounds, which rely on the 
	magnitude of the network parameters, can alleviate the dependence on the 
	width (see \citet{bartlett1998sample,anthony2009neural}). However, most 
	such bounds in the literature have a strong (often exponential) dependence 
	on the network depth. To give one recent example, \citet{neyshabur2015norm} 
	use Rademacher complexity tools to show that if the 
	parameter matrices $W_1,\ldots,W_d$ in each of the $d$ layers have 
	Frobenius 
	norms $\norm{\cdot}_F$ upper-bounded by $M_F(1),\ldots,M_F(d)$ 
	respectively, and under 
	suitable 
	assumptions on the activation functions, the generalization error scales 
	(with 
	high probability) as
	\begin{equation}\label{eq:ney15}
	\Ocal\left(\frac{B2^d\prod_{j=1}^{d}M_F(j)}{\sqrt{m}}\right)~.
	\end{equation}
	Although this bound has no explicit dependence on the network width (that 
	is, 
	the dimensions of $W_1,\ldots,W_d$), it has a very strong, exponential 
	dependence on the network depth $d$, even if $M_F(j)\leq 1$ for all $j$. 
	\citet{neyshabur2015norm} also showed that 
	this dependence can sometimes be avoided for anti-symmetric activations, 
	but 
	unfortunately this is a non-trivial assumption, which is not satisfied for 
	common activations such as the ReLU. \citet{bartlett2017spectrally} use a 
	covering numbers argument to show a bound scaling as
	\begin{equation}
	\label{eq:bartlett2017}
	\tilde{\Ocal}\left(\frac{B\left(\prod_{j=1}^{d}\norm{W_j}\right)
		\left(\sum_{j=1}^{d}\left(\frac{\norm{W_j^T}_{2,1}}{\norm{W_j}}\right)
		^{2/3}\right)^{3/2}}{\sqrt{m}}\right)~,
	\end{equation}
	where $\norm{W}$ denotes the spectral norm of $W$, 
	$\norm{W^T}_{2,1}:=\sum_{l} \sqrt{\sum_k W(l,k)^2}$ denotes the $1$-norm 
	of the 2-norms of the rows of $W$, and where we ignore factors 
	logarithmic in $m$ and the network width. Unlike \eqref{eq:ney15}, here 
	there 
	is no explicit exponential dependence on the depth. However, there is still 
	a 
	strong and unavoidable polynomial dependence: To see this, note that for 
	any 
	$W_j$, $\frac{\sum_{l} \sqrt{\sum_k W_j(l,k)^2}}{\norm{W_j}}\geq 
	\frac{\norm{W_j}_F}{\norm{W_j}}\geq 1$, so the bound above can never be 
	smaller 
	than 
	\[
	\tilde{\Ocal}\left(B\left(\prod_{j=1}^{d}\norm{W_j}\right)
	\sqrt{\frac{d^3}{m}}\right)~.
	\]
	In particular, even if we assume that 
	$B\left(\prod_{j=1}^{d}\norm{W_j}\right)$ is a constant, the bound becomes 
	trivial once $d\geq \Omega(m^{1/3})$.
	Finally, and using the same 
	notation, \citet{neyshabur2017pac} utilize a PAC-Bayesian analysis to prove 
	a 
	bound scaling as
	\begin{equation}\label{eq:neyshabur2017pac}
	\tilde{\Ocal}\left(B\left(\prod_{j=1}^{d}\norm{W_j}\right)\sqrt{\frac{d^2h\sum_{i=1}^{d}
			\frac{\norm{W_j}_F^2}{\norm{W_j}^2}}{m}}\right)~,
	\end{equation}
	where $h$ denotes the network width.\footnote{\cite{bartlett2017spectrally} 
	note that this PAC-Bayesian bound is never better than the bound in  
	\eqref{eq:bartlett2017} derived from a covering numbers argument.}
	Again, since $\frac{\norm{W_j}_F}{\norm{W_j}}$ (the ratio of the Frobenius 
	and spectral norms) is always at least $1$ for any matrix, this bound can 
	never be smaller than 
	$
	\tilde{\Ocal}\left(B\left(\prod_{j=1}^{d}\norm{W_j}\right)\sqrt{\frac{d^3h}{m}}\right)
	$,
	and becomes trivial once $d\sqrt[3]{h}\geq\Omega(m^{1/3})$. To summarize, 
	although some of the bounds above have logarithmic or no dependence on the 
	network width, we are not aware of a bound in the literature which avoids a 
	strong 
	dependence on the depth, even if various norms are controlled. 
	
	Can this depth dependency be avoided, assuming the norms are sufficiently 
	constrained? We argue that in some cases, it 
	must be true. To see this, let us return to the well-understood case of 
	linear 
	predictors, and consider generalized linear predictors of the form
	\[
	\left\{\bx\mapsto 
	\sigma\left(\bw^\top\bx\right)~:~\norm{\bw}\leq
	M\right\}~,
	\]
	where $\sigma(z)=\max\{0,z\}$ is the ReLU function. Like plain-vanilla 
	linear 
	predictors, the generalization error of this class is 
	well-known to be $\Ocal(MB/\sqrt{m})$, assuming the inputs satisfy 
	$\norm{\bx}\leq B$ almost surely. However, it is not difficult to show that 
	this class can be equivalently written as a class of ``ultra-thin'' ReLU 
	networks of 
	the form
	\begin{equation}\label{eq:class}
	\left\{\bx\mapsto \sigma(w_{d}\cdot\sigma(\ldots 
	w_2\cdot\sigma(\bw_1^\top 
	\bx)))~:~\norm{\bw_1}\cdot\prod_{j=2}^{d}|w_j|\leq M\right\}~,
	\end{equation}
	(where $\bw_1$ is a vector and $w_2,\ldots,w_d$ are scalars), where the 
	depth 
	$d$ is arbitrary. Therefore, the sample complexity of this class must also 
	scale as $\Ocal(MB/\sqrt{m})$: This depends on the norm product $M$, but is 
	completely independent of the network depth $d$ 
	as well as the dimension of $\bw_1$. We argue that a ``satisfactory'' 
	sample 
	complexity analysis should have similar independence properties when 
	applied on 
	this class.
	
	In more general neural networks, the vector $\bw_1$ and scalars 
	$w_2,\ldots,w_d$ become matrices $W_1,\ldots,W_d$, and the simple 
	observation 
	above no longer applies. However, using the same intuition, it is natural 
	to 
	try 
	and derive generalization bounds by controlling 
	$\prod_{j=1}^{d}\norm{W_j}$, where $\norm{\cdot}$ is a suitable matrix 
	norm. Perhaps the simplest choice is the spectral norm (and indeed, a 
	product 
	of spectral norms was utilized in some of the previous results mentioned 
	earlier). However, as we formally show in \secref{sec:lowerbound}, the 
	spectral 
	norm alone is too weak to get size-independent bounds, even if the network 
	depth is small. Instead, we show that controlling other suitable norms can 
	indeed lead to better depth dependence, or even fully size-independent 
	bounds, 
	improving on earlier works. Specifically, we make the following 
	contributions:
	\begin{itemize}
		\item In \secref{sec:depthpoly}, we show that the exponential depth 
		dependence in Rademacher complexity-based analysis (e.g. 
		\citet{neyshabur2015norm}) can be avoided by applying contraction to a 
		slightly different object than what has become standard since the work 
		of 
		\cite{bartlett2002rademacher}. For 
		example, for networks with parameter matrices of Frobenius norm at most 
		$M_F(1),\ldots,M_F(d)$, the bound in \eqref{eq:ney15} can be improved 
		to 
		\begin{equation}\label{eq:frobintro}
		\Ocal\left(
		\frac{B\sqrt{d}\prod_{j=1}^d 
			M_F(j)}{\sqrt m}
		\right)~.
		\end{equation}
		The technique can also be applied to other types of norm constraints. 
		For 
		example, if we consider an $\ell_1/\ell_{\infty}$ setup, corresponding 
		to 
		the class of depth-$d$ networks, where the $1$-norm of each row of 
		$W_j$ is 
		at most $M(j)$, we attain a bound of
		\[
		\Ocal\left(\frac{B\sqrt{d+\log(n)}\cdot\prod_{j=1}^{d}M(j)}{\sqrt{m}}\right),
		\]
		where $n$ is the input dimension. Again, the dependence on $d$ is 
		polynomial and quite mild. In contrast, \citet{neyshabur2015norm} 
		studied a 
		similar setup and only managed to obtain an exponential dependence on 
		$d$.
		\item In \secref{sec:depthind}, we develop a generic technique to 
		convert 
		depth-dependent bounds to depth-independent bounds, assuming some 
		control 
		over any Schatten norm of the parameter matrices (which includes, for 
		instance, the 
		Frobenius norm and the trace norm as special cases). The key 
		observation we 
		utilize is that the prediction function computed by such networks can 
		be 
		approximated by the composition of a 
		shallow network and univariate Lipschitz functions. For example, again 
		assuming that the Frobenius norms of the layers are bounded by 
		$M_F(1),\ldots,M_F(d)$, we can 
		further improve \eqref{eq:frobintro} to
		\begin{equation}\label{eq:frobb}
		\tilde{\Ocal}\left(B\left(\prod_{j=1}^d 
		M_F(j)\right)\cdot\min\left\{\sqrt{\frac{\log\left(\frac{1}{\Gamma}
				\prod_{j=1}^{d}M_F(j)\right)}{\sqrt{m}}}~,~
		\sqrt{\frac{d}{m}}\right\}
		\right)~,
		\end{equation}
		where $\Gamma$ is a lower bound on the product of the \emph{spectral} 
		norms of the parameter matrices (note that $\Gamma\leq \prod_{j}M_F(j)$ 
		always). 
		Assuming that $\prod_{j}M_F(j)\leq 
		R$ for some $R$, this can be upper bounded by 
		$\tilde{\Ocal}(BR\sqrt{\log(R/\Gamma)/\sqrt{m}})$, which to the best of 
		our 
		knowledge, is the first explicit bound for standard neural networks 
		which 
		is fully size-independent, assuming only suitable norm constraints. 
		Moreover, it 
		captures the depth-independent sample complexity behavior of the 
		network 
		class in \eqref{eq:class}  discussed earlier. 
		We also apply this technique to get a depth-independent version of the 
		bound in \citep{bartlett2017spectrally}: Specifically, if we assume 
		that the spectral norms satisfy $\norm{W_j}\leq M(j)$ for all $j$, and 
		$\max_j \frac{\norm{W_j^T}_{2,1}}{\norm{W_j}}\leq L$, then the bound in 
		\eqref{eq:bartlett2017} provided by
		\citet{bartlett2017spectrally} becomes
		\[
		\tilde{\Ocal}\left(
		BL\prod_{j=1}^{d}M(j)\cdot\sqrt{\frac{d^3}{m}}\right)~.
		\]
		In contrast, we show the following bound for any $p\geq 1$ (ignoring 
		some 
		lower-order logarithmic factors):
		\[
		\tilde{\Ocal} \left( BL\prod_{j=1}^d M(j)
		\cdot \min \left\{ \frac{\log \left( \frac{1}{\Gamma} \prod_{j=1}^d 
		M_p(j) 
			\right)^{\frac{1}{\frac{2}{3}+p}} 
			}{m^{\frac{1}{2+3p}}},\sqrt{\frac{d^{3}}{ m}} 
		\right\} \right)~,
		\]
		where $M_p(j)$ is an upper bound on the Schatten $p$-norm of $W_j$, and 
		$\Gamma$ is a lower bound on $\prod_{j=1}^{d} \norm{W_j}$. Again, by 
		upper bounding the $\min$ by 
		its first argument, we get a bound independent of the depth $d$, 
		assuming 
		the norms are suitably constrained.
		\item In \secref{sec:lowerbound}, we provide a lower bound, showing 
		that 
		for any $p$, the class of depth-$d$, width-$h$ neural networks, where 
		each 
		parameter matrix $W_j$ has Schatten $p$-norm at most $M_p(j)$, can have 
		Rademacher complexity of at least
		\[
		\Omega\left(
		\frac{B\prod_{j=1}^{d}M_p(j)\cdot 
			h^{\max\left\{0,\frac{1}{2}-\frac{1}{p}\right\}}}{\sqrt{m}}\right)~.
		\]
		This somewhat improves on \citet[Theorem 
		3.6]{bartlett2017spectrally}, which only showed such a result for 
		$p=\infty$ (i.e. with spectral norm control), and without the $h$ term. 
		For 
		$p=2$, it matches the upper bound in \eqref{eq:frobb} in terms of the 
		norm 
		dependencies and $B$. Moreover, it establishes that controlling the  
		spectral norm alone (and indeed, any Schatten $p$-norm control with 
		$p>2$) 
		cannot lead to bounds independent of the size of the network. Finally, 
		the 
		bound shows (similar to \citet{bartlett2017spectrally}) that a 
		dependence 
		on products of norms across layers is generally inevitable.
	\end{itemize}
	
	Besides the above, we provide some additional remarks and observations in 
	\secref{sec:remarks}. Most of our technical proofs are presented in 
	\secref{sec:proofs}.
	
	Finally, we emphasize that the bounds of \secref{sec:depthind} are 
	independent of the network size, only under the assumption that products of 
	norms (or at least ratios of norms) across all layers are bounded by a 
	constant, which is quite restrictive in practice. For example, it is enough 
	that the Frobenius norm of each layer matrix is at least $2$, to get that 
	\eqref{eq:frobb} scales as $2^d$ where $d$ is the number of layers. 
	However, our focus here is to show that at least \emph{some} norm-based 
	assumptions lead to size-independent bounds, and hope these can be weakened 
	in future work.

	\section{Preliminaries}\label{sec:prelim}
	
	\textbf{Notation.} We use bold-faced letters to denote vectors, and capital 
	letters to denote matrices or fixed 
	parameters (which should be clear from context). Given a 
	vector $\bw \in \reals^h$, $\norm{\bw}$ will refer to the Euclidean norm, 
	and for $p \geq 1$, $\| \bw \|_p = \left( \sum_{i=1}^h |\bw_i|^p 
	\right)^{1/p}$ will refer to the $\ell_p$ norm. For a matrix $W$, 
	we use
	$\norm{W}_p$, where $p\in [1,\infty]$, to denote the Schatten $p$-norm 
	(that 
	is, 
	the $p$-norm of the spectrum of $W$, written as a vector). For example, 
	$p=\infty$ 
	refers to the 
	spectral norm, $p=2$ refers to the Frobenius norm, and 
	$p=1$ refers to the trace norm. For the case of the spectral norm, 
	we will drop the $\infty$ subscript, and use just $\norm{W}$. Also, in 
	order to follow standard convention, we use $\norm{W}_F$ to denote the 
	Frobenius norm. Finally, given a matrix W and reals $p,q \geq 1$, we let 
	$\norm{W}_{p,q}:= \left( \sum_k \left(\sum_{j}|W_{j,k}|^p \right)^{q/p} 
	\right)^{1/q}$ denote the $q$-norm of the $p$-norms of the columns of $W$.
	
	\textbf{Neural Networks.}
	Given the domain $\Xcal=\{\bx:\norm{\bx}\leq B\}$ in Euclidean space, we 
	consider  
	(scalar or vector-valued) standard neural networks, of the form
	\[
	\bx\mapsto W_d\sigma_{d-1}(W_{d-1}\sigma_{d-2}(\ldots \sigma_1(W_1\bx))),
	\]
	where each $W_j$ is a parameter matrix, and 
	each 
	$\sigma_j$ is some fixed Lipschitz continuous function between Euclidean 
	spaces, satisfying 
	$\sigma_j(\mathbf{0})=\mathbf{0}$. In the above, we denote $d$ as the depth 
	of 
	the network, and its width $h$ is defined as the maximal row or column 
	dimensions of $W_1,\ldots,W_d$. Without loss of generality, we will assume 
	that 
	$\sigma_j$ has a Lipschitz constant of at most $1$ (otherwise, the 
	Lipschitz constant can be absorbed into the norm constraint of the 
	neighboring 
	parameter matrix). We say that $\sigma$ is element-wise if it can be 
	written as 
	an application of the same univariate function over each coordinate of its 
	input (in which case, somewhat abusing notation, we will also use $\sigma$ 
	to 
	denote that univariate function). We say that $\sigma$ is 
	positive-homogeneous 
	if it is element-wise and satisfies $\sigma(\alpha z) = \alpha \sigma(z)$ 
	for 
	all $\alpha\geq 0$ and $z\in\reals$. An important example of the above are 
	ReLU 
	networks, where every $\sigma_j$ corresponds to applying the 
	(positive-homogeneous) ReLU function $\sigma(z)=\max\{0,z\}$ on each 
	element. 
	To simplify notation, we let $W_b^r$ be shorthand for the matrix tuple 
	$\{W_b,W_{b+1},\ldots,W_r\}$, and $N_{W_b^r}$ denote the function computed 
	by 
	the sub-network composed of layers $b$ through $r$, that is
	\[
	\bx\mapsto W_r\sigma_{r-1}(W_{r-1}\sigma_{r-2}(\ldots \sigma_b(W_b\bx)))~.
	\]
	
	\textbf{Rademacher Complexity.} The results in this paper focus on 
	Rademacher 
	complexity, which is a standard tool to control the uniform convergence 
	(and 
	hence the sample complexity) of given classes of predictors (see 
	\citet{bartlett2002rademacher,shalev2014understanding} for more details). 
	Formally, given a real-valued function class $\Hcal$ and some set of data 
	points 
	$\bx_1,\ldots,\bx_m\in \Xcal$, we define the (empirical) Rademacher 
	complexity 
	$\hat{\Rcal}_m(\Hcal)$ as 
	\begin{equation}\label{eq:raddef}
	\hat{\Rcal}_m(\Hcal) = 
	\E_{\boldsymbol{\epsilon}}\left[\sup_{h\in\Hcal}\frac{1}{m}\sum_{i=1}^{m}
	\varepsilon_i h(\bx_i)\right],
	\end{equation}
	where $\boldsymbol{\varepsilon}=(\varepsilon_1,\ldots,\varepsilon_m)$ is a 
	vector 
	uniformly distributed in $\{-1,+1\}^m$. Our main results provide bounds 
	on the Rademacher complexity (sometimes independent of 
	$\bx_1,\ldots,\bx_m$, as 
	long as they are assumed to have norm at most $B$), with respect to classes 
	of 
	neural networks with various norm constraints. Using standard arguments, 
	such 
	bounds can be converted to bounds on the generalization error, assuming 
	access 
	to a sample of $m$ i.i.d. training examples.

	\section{From Exponential to Polynomial Depth Dependence}
	\label{sec:depthpoly}
	
	To get bounds on the Rademacher complexity of deep neural networks, a 
	reasonable approach (employed in \citet{neyshabur2015norm}) is to use a 
	``peeling'' argument, where the complexity bound for depth $d$ networks is 
	reduced to a complexity bound for depth $d-1$ networks, and then applying 
	the 
	reduction $d$ times. For example, consider the class $\Hcal_d$ of depth-$d$ 
	ReLU real-valued neural networks, with each layer's parameter matrix with 
	Frobenius norm at most $M_F(j)$. 
	Using some straightforward manipulations, it is possible to show that 
	$\hat{\Rcal}_m(\Hcal_d)$, which by definition equals 
	\[
	\E_{\boldsymbol{\epsilon}}\sup_{h\in\Hcal_d}\frac{1}{m}\sum_{i=1}^{m}\epsilon_i
	h(\bx_i)
	~=~
	\E_{\boldsymbol{\epsilon}}\sup_{h\in \Hcal_{d-1}}\sup_{W_d : \| W_d \|_F 
	\leq 
		M_F(d)}\frac{1}{m}\sum_{i=1}^{m}\epsilon_i W_d 
	\sigma(h(\bx_i)),
	\]
	can be upper bounded by
	\begin{equation}\label{eq:dexp}
	M_F(d)\cdot 
	\E_{\boldsymbol{\epsilon}}\sup_{h\in\Hcal_{d-1}}\left\|\frac{1}{m}\sum_{i=1}^{m}\epsilon_i
	\sigma(h(\bx_i))
	\right\|~\leq~
	2M_F(d)\cdot 
	\E_{\boldsymbol{\epsilon}}\sup_{h\in\Hcal_{d-1}}\left\|\frac{1}{m}\sum_{i=1}^{m}\epsilon_i
	h(\bx_i))
	\right\|~.
	\end{equation}
	Iterating this inequality $d$ times, one ends up with a bound scaling as 
	$2^d\prod_{j=1}^{d}M_F(j)$ (as in \citet{neyshabur2015norm}, see also 
	\eqref{eq:ney15}). The exponential $2^d$ factor follows from the $2$ factor 
	in  
	\eqref{eq:dexp}, which in turn follows from applying the Rademacher 
	contraction 
	principle to get rid of the $\sigma$ function. Unfortunately, this $2$ 
	factor is
	generally unavoidable (see the discussion in \citet{ledoux2013probability} 
	following Theorem 4.12).
	
	In this section, we point out a simple trick, which can be used to reduce 
	such 
	exponential depth dependencies to polynomial ones. In a nutshell, using 
	Jensen's inequality, we can rewrite the (scaled) Rademacher complexity 
	$m\cdot\hat{\Rcal}_m(\Hcal)=\E_{\boldsymbol{\epsilon}}\sup_{h\in\Hcal}\sum_{i=1}^{m}\epsilon_i
	h(\bx_i)$ as
	\[
	\frac{1}{\lambda}\log\exp\left(\lambda\cdot 
	\E_{\boldsymbol{\epsilon}}\sup_{h\in\Hcal}
	\sum_{i=1}^{m}\epsilon_i h(\bx_i)
	\right)
	~\leq~
	\frac{1}{\lambda}\log\left(\E_{\boldsymbol{\epsilon}}\sup_{h\in\Hcal}\exp\left(\lambda
	\sum_{i=1}^{m}\epsilon_i h(\bx_i)
	\right)\right),
	\]
	where $\lambda>0$ is an arbitrary parameter. We then perform a ``peeling'' 
	argument similar to before, resulting in a multiplicative $2$ factor after 
	every peeling step. Crucially, these factors accumulate \emph{inside} the 
	log 
	factor, so that the end result contains only a $\log(2^d)=d$ factor, which 
	by 
	appropriate tuning of $\lambda$, can be further reduced to $\sqrt{d}$.
	
	The formalization of this argument depends on the matrix norm we are using, 
	and 
	we will begin with the case of the Frobenius norm. A key technical 
	condition 
	for the argument to work is that we can perform the ``peeling'' inside the 
	$\exp$ function. This is captured by the following lemma:
	\begin{lemma}
		\label{lem:contraction}
		Let $\sigma$ be a $1$-Lipschitz, positive-homogeneous activation 
		function 
		which is applied element-wise (such as the ReLU). Then for any 
		class of vector-valued functions $\Fcal$, 
		and 
		any convex and monotonically increasing function 
		$g:\reals\rightarrow[0,\infty)$,
		\[
		\E_{\boldsymbol{\epsilon}}\sup_{f\in\Fcal,W:\norm{W}_F\leq 
			R}g\left(\left\|\sum_{i=1}^{m}
		\epsilon_i\sigma(Wf(\bx_i))\right\|\right)~\leq~
		2\cdot\E_{\boldsymbol{\epsilon}}\sup_{f\in\Fcal}g\left(R\cdot 
		\left\|\sum_{i=1}^{m}\epsilon_i 
		f(\bx_i)\right\|\right)~.
		\]
	\end{lemma}
	\begin{proof}
		Letting $\bw_1,\bw_2,\ldots,\bw_h$ be the rows of the matrix $W$, we 
		have
		\begin{align*}
		\left\|\sum_{i=1}^{m}
		\epsilon_i\sigma(Wf(\bx_i))\right\|^2&=
		\sum_{j=1}^{h}\norm{\bw_j}^2\left(\sum_{i=1}^{m}
		\epsilon_i\sigma\left(\frac{\bw_j^\top}{\norm{\bw_j}} 
		f(\bx_i)\right)\right)^2~.\\
		\end{align*}
		The supremum of this over all $\bw_1,\ldots,\bw_h$ such that 
		$\norm{W}_F^2=\sum_{j=1}^{h}\norm{\bw_j}^2\leq R^2$ must be attained 
		when 
		$\norm{\bw_j}=R$ for some $j$, and $\norm{\bw_i}=0$ for all $i\neq j$. 
		Therefore,
		\begin{align*}
		\E_{\boldsymbol{\epsilon}}\sup_{f\in\Fcal,W:\norm{W}_F\leq 
			R}g\left(\left\|\sum_{i=1}^{m}
		\epsilon_i\sigma(Wf(\bx_i))\right\|\right)&=
		\E_{\boldsymbol{\epsilon}}\sup_{f\in\Fcal,\bw:\norm{\bw}=R}g\left(\left|\sum_{i=1}^{m}
		\epsilon_i\sigma(\bw^\top f(\bx_i))\right|\right)~.
		\end{align*}
		Since $g(|z|)\leq g(z)+g(-z)$, this can be upper bounded by
		\[
		\E_{\boldsymbol{\epsilon}}\sup 
		g\left(\sum_{i=1}^{m}\epsilon_i\sigma(\bw^\top f(\bx_i))\right)+
		\E_{\boldsymbol{\epsilon}}\sup 
		g\left(-\sum_{i=1}^{m}\epsilon_i\sigma(\bw^\top f(\bx_i))\right)
		~=~
		2\cdot\E_{\boldsymbol{\epsilon}}\sup 
		g\left(\sum_{i=1}^{m}\epsilon_i\sigma(\bw^\top 
		f(\bx_i))\right),
		\]
		where the equality follows from the symmetry in the distribution of the 
		$\epsilon_i$ random variables. The right hand side in turn can be upper 
		bounded by
		\begin{align*}
		2\cdot\E_{\boldsymbol{\epsilon}}\sup_{f\in\Fcal,\bw:\norm{\bw}=R} 
		g\left(\sum_{i=1}^{m}\epsilon_i 
		\bw^\top f(\bx_i)\right)&~\leq~
		2\cdot\E_{\boldsymbol{\epsilon}}\sup_{f\in\Fcal,\bw:\norm{\bw}=R} 
		g\left(\norm{\bw}\left\|\sum_{i=1}^{m}\epsilon_i 
		f(\bx_i)\right\|\right)\\
		&~=~
		2\cdot\E_{\boldsymbol{\epsilon}} \sup_{f\in\Fcal}g\left(R\cdot 
		\left\|\sum_{i=1}^{m}\epsilon_i 
		f(\bx_i)\right\|\right)~.
		\end{align*}
		(see equation 4.20 in \citet{ledoux2013probability}).
	\end{proof}
	
	With this lemma in hand, we can provide a bound on the Rademacher 
	complexity of 
	Frobnius-norm-bounded neural networks, which is as clean as 
	\eqref{eq:ney15}, 
	but where the $2^d$ factor is replaced by $\sqrt{d}$:
	\begin{theorem}
		\label{thm:sqrtl}
		Let $\Hcal_d$ be the class of real-valued networks of depth $d$ over 
		the 
		domain $\Xcal$, where  each parameter matrix $W_j$ has Frobenius norm 
		at 
		most $M_F(j)$, and with activation functions satisfying
		\lemref{lem:contraction}. Then
		\[
		\hat{\Rcal}_m(\Hcal_d) \leq \frac 1m \prod_{j=1}^d M_F(j) \cdot  
		\left(\sqrt{2\log(2)d}+1\right)\sqrt{\sum_{i=1}^m \|\bx_i\|^2} \leq 
		\frac{B\left(\sqrt{2\log(2)d}+1\right)\prod_{j=1}^d 
			M_F(j)}{\sqrt m}.
		\]
	\end{theorem}
	\begin{proof}
		Fix $\lambda>0$, to be chosen later. The Rademacher complexity can be 
		upper 
		bounded as
		\begin{align*}
		m\hat{\Rcal}_m(\Hcal_d) &= \E_{\boldsymbol{\epsilon}} 
		\sup_{N_{W_1^{d-1}},W_d} \sum_{i=1}^m 
		\epsilon_i W_d \sigma_{d-1}(N_{W_1^{d-1}}(\bx_i))  \\
		&\leq \frac{1}{\lambda} \log \E_{\boldsymbol{\epsilon}} \sup 
		\exp\left( \lambda \sum_{i=1}^m 
		\epsilon_i W_d \sigma_{d-1}(N_{W_1^{d-1}}(\bx_i))\right) \\
		&\leq \frac{1}{\lambda} \log \E_{\boldsymbol{\epsilon}} \sup \exp  
		\left( M_F(d)\cdot 
		\left\|\lambda \sum_{i=1}^m \epsilon_i 
		\sigma_{d-1}(N_{W_1^{d-1}}(\bx_i)) \right\| \right)
		\end{align*}
		We write this last expression as
		\begin{align*}
		&\frac{1}{\lambda} \log \E_{\boldsymbol{\epsilon}} \sup_{f, 
			\|W_{d-1}\|_F \leq M_F(d-1)} 
		\exp  \left( M_F(d)\cdot \lambda \left\| \sum_{i=1}^m \epsilon_i 
		\sigma_{d-1}(W_{d-1} f(\bx_i)) \right\| \right)\\
		&\leq \frac{1}{\lambda} \log \left( 2\cdot 
		\E_{\boldsymbol{\epsilon}} \sup_{f} \exp \left( 
		M_F(d)\cdot M_F(d-1) \cdot \lambda \left\| \sum_{i=1}^m \epsilon_i 
		f(\bx_i) \right\| \right) \right)
		\end{align*}
		where $f$ ranges over all possible functions $\sigma_{d-2}\circ 
		N_{W_1^{d-2}} (\bx)$. Here we applied Lemma~\ref{lem:contraction} with 
		$g(\alpha) = \exp\{M_F(d)\lambda \cdot \alpha\}$. Repeating the 
		process, we 
		arrive at
		\begin{align}
		m\hat{\Rcal}_m(\Hcal_d) &\leq \frac{1}{\lambda} \log \left(2^d \cdot 
		\E_{\boldsymbol{\epsilon}} 
		\exp \left( M \lambda \left\| \sum_{i=1}^m \epsilon_i \bx_i \right\| 
		\right) \right)
		\label{eq:softmax_bd}
		\end{align}
		where $M= \prod_{j=1}^d M_F(j)$.
		Define a random variable
		$$Z = M\cdot\left\|\sum_{i=1}^{m}\epsilon_i \bx_i\right\|,$$
		(random as a function of the random variables 
		$\epsilon_1,\ldots,\epsilon_m$). Then
		\begin{align}
		\label{eq:softmax_bd2}
		\frac{1}{\lambda} \log \left\{ 2^d \cdot \E \exp \lambda Z \right\} 
		&= \frac{d\log(2)}{\lambda} +  \frac{1}{\lambda} \log \left\{ \E 
		\exp 
		\lambda (Z-\E Z) \right\} + \E Z.
		\end{align}
		By Jensen's inequality, $\E[Z]$ can be upper bounded by
		\[
		M\sqrt{\E_{\boldsymbol{\epsilon}}\left[\left\|\sum_{i=1}^{m}\epsilon_i\bx_i\right\|^2\right]}
		~=~
		M\sqrt{\E_{\boldsymbol{\epsilon}}\left[\sum_{i,i'=1}^{m}\epsilon_i\epsilon_{i'}\bx_i^\top\bx_{i'}\right]}
		~=~
		M\sqrt{\sum_{i=1}^{m}\norm{\bx_i}^2}.
		\]
		To handle the $ \log \left\{ \E 
		\exp 
		\lambda (Z-\E Z) \right\}$ term in \eqref{eq:softmax_bd2}, note that 
		$Z$ is 
		a deterministic function of the i.i.d. random variables 
		$\epsilon_1,\ldots,\epsilon_m$, and satisfies
		\begin{align*}
		Z(\epsilon_1,\ldots,\epsilon_i, \ldots,\epsilon_m) - 
		Z(\epsilon_1,\ldots,-\epsilon_i, \ldots,\epsilon_m) \leq 2M \|\bx_i\|~.
		\end{align*}
		This means that $Z$ satisfies a bounded-difference condition, which by 
		the 
		proof of Theorem 6.2 in \citep{boucheron2013concentration}, implies 
		that 
		$Z$ is sub-Gaussian, with variance factor
		$$v = \frac{1}{4} \sum_{i=1}^m (2M\|\bx_i\|)^2 = M^2\sum_{i=1}^m 
		\|\bx_i\|^2,$$
		and satisfies
		$$ \frac{1}{\lambda} \log \left\{ \E \exp \lambda (Z-\E Z) \right\} 
		\leq 
		\frac{1}{\lambda} \frac{\lambda^2 M^2 \sum_{i=1}^m \|\bx_i\|^2}{2} =  
		\frac{\lambda M^2  \sum_{i=1}^m \|\bx_i\|^2}{2}.$$
		Choosing $\lambda = \frac{\sqrt{2\log(2)d}}{M \sqrt{\sum_{i=1}^m \| 
		\bx_i 
				\|^2}}$ 
		and using the above, we get that \eqref{eq:softmax_bd} can be upper 
		bounded 
		as follows:
		\begin{align*}
		\frac{1}{\lambda} \log \left\{ 2^d \cdot \E \exp \lambda Z \right\} 
		&\leq \E Z + \sqrt{2\log(2)d} \cdot M\sqrt{\sum_{i=1}^m \|\bx_i\|^2} 
		\leq M 
		\left(\sqrt{2\log(2)d}+1\right)\sqrt{\sum_{i=1}^m \|\bx_i\|^2}~,
		\end{align*}
		from which the result follows.
	\end{proof}
	
	\begin{remark}
		We note that for simplicity, the bound in \thmref{thm:sqrtl} is stated 
		for 
		real-valued networks, but 
		the argument easily carries over to vector-valued networks, composed 
		with some 
		real-valued Lipschitz loss function. In that case, one uses a variant 
		of 
		\lemref{lem:contraction} to peel off the losses, and then proceed in 
		the same 
		manner as in the proof of \thmref{thm:sqrtl}. We omit the precise 
		details for 
		brevity.
	\end{remark}
	
	A result similar to the above can also be derived for other  
	matrix norms. For example, given a matrix $W$, let $\norm{W}_{1,\infty}$ 
	denote the maximal $1$-norm of its rows, and consider the class $\Hcal_d$ 
	of depth-$d$ networks, where each parameter matrix $W_j$ satisfies 
	$\norm{W_j}_{1,\infty}\leq M(j)$ for all $j$ (this corresponds to a 
	setting, also studied in \citet{neyshabur2015norm}, where the $1$-norm of 
	the weights of each neuron in the network is bounded). In this case, we can 
	derive a variant 
	of \lemref{lem:contraction}, which in fact does not require 
	positive-homogeneity of the activation function:
	\begin{lemma}
		\label{lem:contraction_l1}
		Let $\sigma$ be a $1$-Lipschitz  activation function with $\sigma(0)=0$,
		applied element-wise. Then for any vector-valued class $\Fcal$, 
		and 
		any convex and monotonically increasing function 
		$g:\reals\rightarrow[0,\infty)$,
		\begin{equation}
		\label{eq:contr_lemma2_statement}
		\E_{\boldsymbol{\epsilon}}\sup_{f\in\Fcal,W:\norm{W}_{1,\infty}\leq 
			R}g\left(\left\|\sum_{i=1}^{m}
		\epsilon_i\sigma(Wf(\bx_i))\right\|_\infty\right)~\leq~
		2\cdot\E_{\boldsymbol{\epsilon}}\sup_{f\in\Fcal}g\left(R\cdot 
		\left\|\sum_{i=1}^{m}\epsilon_i 
		f(\bx_i)\right\|_\infty\right)~,
		\end{equation}
		where $\norm{\cdot}_{\infty}$ denotes the vector infinity norm.
	\end{lemma}
	
	Using the same technique as before, we can use this lemma to get a bound on 
	the 
	Rademacher complexity for $\Hcal_d$:
	\begin{theorem}
		\label{thm:sqrtl1}
		Let $\Hcal_d$ be the class of real-valued networks of depth $d$ over 
		the 
		domain $\Xcal$, where $\norm{W_j}_{1,\infty}\leq M(j)$ for all 
		$j\in\{1,\ldots,d\}$, and with activation functions satisfying the 
		condition of \lemref{lem:contraction_l1}. Then
		\begin{align*}
		\hat{\Rcal}_m(\Hcal_d) \leq     
		\frac{2}{m}\prod_{j=1}^{d}M(j)\cdot
		\sqrt{d+1+\log(n)}\cdot\sqrt{\max_{j\in\{1,\ldots,n\}}\sum_{i=1}^{m}x_{i,j}^2}
		~\leq~ \frac{2B\sqrt{d+1+\log(n)}\cdot\prod_{j=1}^{d}M(j)}{\sqrt{m}}~,
		\end{align*}
		where $x_{i,j}$ is the $j$-th coordinate of the vector $\bx_i$. 
	\end{theorem}
	The proofs of the theorem, as well as \lemref{lem:contraction_l1}, appear 
	in 
	\secref{sec:proofs}. 
	
	The constructions used in the results of this section use the function 
	$g(z)=\exp(\lambda z)$ together with its inverse $g^{-1}(z)= 
	\frac{1}{\lambda}\log(z)$, to get depth dependencies 
	scaling as $g^{-1}(2^d)$. Thus, it might be tempting to try 
	and further improve the depth dependence, by using other functions $g$ for 
	which $g^{-1}$ increases sub-logarithmically. Unfortunately, the argument 
	still requires us to control  $\E_{\boldsymbol{\epsilon}} 
	g\left(\left\|\sum_{i=1}^{m}\epsilon_i \bx_i\right\|\right)$, which is 
	difficult if $g$ increases more than exponentially. In the next section, we 
	introduce a different idea, which under 
	suitable assumptions, allows us to get rid of the depth dependence 
	altogether.

	\section{From Depth Dependence to Depth Independence}
	\label{sec:depthind}

	In this section, we develop a general result, which allows one to convert 
	any 
	depth-dependent bound on the Rademacher complexity of neural networks, to a
	depth-independent one, assuming that the Schatten $p$-norms of the 
	parameter 
	matrices (for any $p\in [1,\infty)$) is sufficiently controlled. We develop 
	and 
	formalize 
	the main 
	result in \subsecref{subsec:general}, and provide applications in 
	\subsecref{subsec:applications}. The proofs the results in this section  
	appear 
	in \secref{sec:proofs}. 
	
	\subsection{A General Result}\label{subsec:general}
	
	To motivate our approach, let us consider a special case of depth-$d$ 
	networks, 
	where 
	\begin{itemize}
		\item Each parameter matrix $W_1,\ldots,W_{d-1}$ is constrained to be 
		diagonal and of size $h\times h$.
		\item The Frobenius norm of every $W_1,\ldots,W_d$ 
		is at most $1$. 
		\item All activation functions are the identity (so the network 
		computes a 
		linear function).
	\end{itemize}
	Letting 
	$\bw_i$ be 
	the diagonal of $W_i$, such networks are equivalent to 
	\[
	\bx\mapsto (\bw_d\circ \bw_{d-1}\circ\ldots\circ \bw_1)^\top \bx~,
	\]
	where $\circ$ denotes element-wise product.
	Therefore, if we would like the network to compute a non-trivial function, 
	we 
	clearly need that $\bw_d\circ\ldots\circ \bw_1$ be bounded away from zero 
	(e.g., not exponentially small in $d$), while still satisfying the 
	constraint 
	$\norm{\bw_j} \leq 1$ for all $j$. In fact, the only way to satisfy both 
	requirements simultaneously is if $\bw_1,\ldots,\bw_d$ are all close to 
	some 
	1-sparse unit vector, which implies that the matrices $W_1,\ldots,W_d$ must 
	be 
	close to being rank-1.
	
	It turns out that this intuition holds much more generally, even if we do 
	not 
	restrict ourselves to identity activations and diagonal parameter matrices 
	as 
	above. Essentially, what we can show is that if some network computes a 
	non-trivial function, and the 
	product of its Schatten $p$-norms (for any $p<\infty$) is bounded, then 
	there 
	must be at least one parameter matrix, which is not far from being rank-1. 
	Therefore, if we replace that parameter matrix by an appropriate rank-1 
	matrix, 
	the function computed by the network does not change much. This is captured 
	by 
	the following theorem:
	
	\begin{theorem}\label{thm:rankonereplace}
		For any $p\in [1,\infty)$, any network $N_{W_1^d}$ such that 
		$\prod_{j=1}^{d}\norm{W_j}\geq \Gamma$ and 
		$\prod_{j=1}^{d}\norm{W_j}_p\leq 
		M$, and for any 
		$r\in\{1,\ldots,d\}$, there exists 
		another network $N_{\tilde{W}_1^d}$ (of the same depth and layer 
		dimensions) with the following 
		properties:
		\begin{itemize}
			\item $\tilde{W}_1^d=\{\tilde{W}_1,\ldots,\tilde{W}_d\}$ is 
			identical 
			to $W_1^d$, except for the parameter matrix $\tilde{W}_{r'}$ in the $r'$-th layer, for some $r' \in \{1, 2, \ldots, r\}$. The matrix $\tilde{W}_{r'}$ 
			is of rank at most 1, and equals $s\bu\bv^\top$ 
			where $s,\bu,\bv$ are some leading singular value and singular 
			vectors 
			pairs of $W_{r'}$.
			\item $
			\sup_{\bx\in\Xcal}\norm{N_{W_1^d}(\bx)-N_{\tilde{W}_1^d}(\bx)}~\leq~
			B\left(\prod_{j=1}^{d}\norm{W_j}\right)\left(\frac{2p\log(M/\Gamma)}{r}\right)^{1/p}.
			$
		\end{itemize}
	\end{theorem}
	
	We now make the following crucial observation: A real-valued network with a 
	rank-1 
	parameter matrix $W_{r'}=s\bu\bv^\top$ computes the function
	\[
	\bx\mapsto W_d \sigma_{d-1}(\ldots \sigma_{r}(s\bu\bv^\top 
	\sigma_{r'-1}(\ldots 
	\sigma_1(W_1\bx)\ldots)))~.
	\]
	This can be seen as the composition of the \emph{depth-$r'$} network
	\[
	\bx\mapsto \bv^\top \sigma_{r'-1}(\ldots \sigma_1(W_1\bx)\ldots),
	\]
	and the \emph{univariate} function 
	\[
	x\mapsto W_d\sigma_{d-1}(\ldots \sigma_{r'}(s\bu 
	x))~.
	\]
	Moreover, the norm constraints imply that the latter function is 
	Lipschitz. Therefore, the class of networks we are considering is a subset 
	of 
	the class of depth-$r'$ networks composed with  univariate Lipschitz 
	functions. 
	Fortunately, given \emph{any} class with bounded complexity, one can 
	effectively bound the Rademacher complexity of its composition with 
	univariate 
	Lipschitz functions, as formalized in the following theorem.
	%\footnote{For technical reasons, the theorem uses the empirical Gaussian 
	%complexity $\hat{\Gcal}_m$  \citep{bartlett2002rademacher}, which has the 
	%same 
	%definition as the empirical Rademacher complexity $\hat{\Rcal}_m$ (see 
	%\eqref{eq:raddef}), except that the random variables 
	%$\epsilon_1,\ldots,\epsilon_m$ are i.i.d. standard Gaussian. Note that 
	%$\hat{\Gcal}_m$ and $\hat{\Rcal}_m$ are equivalent up to logarithmic 
	%factors 
	%-- 
	%see \citep{ledoux2013probability}.}
	\begin{theorem}\label{thm:radlipschitz}
		Let $\Hcal$ be a class of functions from Euclidean space to $[-R,R]$. 
		Let 
		$\Fcal_{L,a}$ 
		be 
		the class of of $L$-Lipschitz functions from $[-R,R]$ to $\reals$, such 
		that 
		$f(0)=a$ for some fixed $a$. Letting
		$\Fcal_{L,a}\circ\Hcal:=\{f(h(\cdot)):f\in\Fcal_{L,a},h\in\Hcal\}$, its 
		Rademacher complexity satisfies
		\[
		\hat{\Rcal}_m(\Fcal_{L,a}\circ\Hcal)~\leq~
		cL\left(\frac{R}{\sqrt{m}}+\log^{3/2}(m)\cdot\hat{\Rcal}_m(\Hcal)\right)~,
		\]
		where $c>0$ is a universal constant.
	\end{theorem}
	\begin{remark}
		The $\log^{3/2}(m)\cdot \hat{\Rcal}_m(\Hcal)$ can be replaced by 
		$\log(m)\cdot 
		\hat{\Gcal}_m(\Hcal)$, where $\hat{\Gcal}_m(\Hcal)$ is the empirical 
		Gaussian 
		complexity of $\Hcal$ -- see the proof in \secref{sec:proofs} for 
		details.
	\end{remark}
	
	Combining these ideas, our plan of attack is the following: Given some 
	class of 
	depth-$d$ networks, and an arbitrary parameter $r\in \{1,\ldots,d\}$, we 
	use 
	\thmref{thm:rankonereplace} to relate their 
	Rademacher complexity to the complexity of similar networks, but where for some $1 \leq r' \leq r$, the 
	$r'$-th parameter matrix is of rank-1. We then use \thmref{thm:radlipschitz} 
	to 
	bound that complexity in turn using the Rademacher complexity of depth-$r'$ 
	networks. 
	Crucially, the resulting bound has no explicit dependence on the original 
	depth 
	$d$, only on the new parameter $r$. Formally, we have the following 
	theorem, 
	which is the main result of this 
	section: 
	
	\begin{theorem}\label{thm:depthreduc}	
		Consider the following hypothesis class of networks on $\Xcal = 
		\{\bx:\norm{\bx}\leq B\}$:
		\[
		\Hcal=\left\{
		N_{W_1^d}~:~\begin{matrix}\prod_{j=1}^{d}\norm{W_j}\geq 
		\Gamma\\
		\forall j\in \{1\ldots d\},~~
		W_j\in\Wcal_j,~\max\left\{\frac{\norm{W_j}}{M(j)},\frac{\norm{W_j}_p}{M_p(j
			)}\right\}\leq 1\end{matrix}\right\}~,
		%\max_{j=1\ldots,d}\max\left\{\frac{\norm{W_j}}{M(j)},\frac{\norm{W_j}_p}{M_p(j
		%	)}\right\}\leq 1\right\}~.
		\]
		for some parameters $p,\Gamma \geq 1,\{M(j),M_p(j),\Wcal_j\}_{j=1}^{d}$.
		Also, for any $r\in\{1,\ldots,d\}$, define 
		\[
		\Hcal_r =\left\{N_{W_1^r}~:~\begin{matrix}N_{W_1^r}\text{ maps to 
			$\reals$}\\
		\forall j\in\{1\ldots 
		r-1\},~~W_j\in\Wcal_j\\
		\forall j\in\{1\ldots 
		r\},~~\max\left\{\frac{\norm{W_j}}{M(j)},\frac{\norm{W_j}_p}{M_p(j
			)}\right\}\leq 1
		\end{matrix}\right\}~.
		\]
		Finally, for $m>1$, let $\ell\circ\Hcal = 
		\{(\ell_1(h(\bx_1)),\ldots,\ell_m(h(\bx_m)))~:~h\in\Hcal\}$, where 
		$\ell_1,\ldots,\ell_m$ 
		are real-valued loss functions which are $\frac{1}{\gamma}$-Lipschitz 
		and 
		satisfy 
		$\ell_1(\mathbf{0})=\ell_2(\mathbf{0})=\ldots=\ell_m(\mathbf{0})=a$, for some $a \in \reals$. Assume that $|a| \leq \frac{B\prod_{j=1}^d M(j)}{\gamma}$.
		
		Then the Rademacher complexity 
		$\hat{\Rcal}_m(\ell\circ\Hcal)$ is 
		upper bounded by
		% {\small{\color{red}\[
		% \frac{cB\prod_{j=1}^{d}M(j)}{\gamma}
		% \left(
		% \min_{r\in\{1,\ldots,d\}}\left\{\frac{\log^{3/2}(m)\cdot\max_{r' \in \{1, \ldots, r\}}\hat{\Rcal}_m(\Hcal_{r'})}{B\prod_{j=1}^{r}{M(j)}}+[something]+
		% 
		%\left(\frac{\log\left(\frac{1}{\Gamma}\prod_{j=1}^{d}M_p(j)\right)}{r}\right)^{1/p}\right\}+
		% \frac{1}{\sqrt{m}}\right)~,
		% \]}}
		{\small\[
		\frac{cB\prod_{j=1}^{d}M(j)}{\gamma}
		\left(\min_{r \in \{1, \ldots, d\}} \left\{\frac{\log^{3/2}(m)}{B}\cdot \max_{r' \in \{1, \ldots, r\}}\frac{\hat{\Rcal}_m(\Hcal_{r'})}{\prod_{j=1}^{r'}M(j)}+
		\left(\frac{\log\left(\frac{1}{\Gamma}\prod_{j=1}^{d}M_p(j)\right)}{r}\right)^{1/p}+
		\frac{1+\sqrt{\log r}}{\sqrt{m}}\right\}\right),
		\]}
		where $c>0$ is a universal constant.
	\end{theorem}
	
	In particular, one can upper bound this result by any choice of $r$ in 
	$\{1,\ldots,d\}$. By tuning $r$ appropriately, we can get bounds 
	independent of 
	the depth $d$. In the next subsection, we provide some concrete 
	applications 
	for specific choices of $\Hcal$.
	
	\begin{remark}
		The parameters $\Gamma$ and $\gamma$, which divide the norm terms in 
		\thmref{thm:depthreduc}, are both closely related to the notion of a 
		margin. 
		Indeed, if we consider binary or multi-class classification, then 
		bounds as 
		above w.r.t. $\frac{1}{\gamma}$-Lipschitz losses can be converted to a 
		bound on 
		the misclassification error rate in terms of the average 
		$\gamma$-margin error 
		on the training data (see \citet[Section 3.1]{bartlett2017spectrally} 
		for a 
		discussion). Also, $B\Gamma$ can be viewed as the ``maximal'' 
		margin 	attainable over the input domain, since 
		$\sup_{\bx\in\Xcal}\norm{N_{W_1^d}(\bx)}\leq 
		B\prod_{j=1}^{d}\norm{W_j}=B\Gamma$.
	\end{remark}

	\subsection{Applications of 
	\thmref{thm:depthreduc}}\label{subsec:applications}

	%\citet{neyshabur2015norm} provides the following result for networks with 
	%parameter matrices bounded in terms of the Frobenius norm:
	%\begin{theorem}\label{thm:neyshabur15}
	%	\citep[Theorem 1]{neyshabur2015norm}
	%	The hypothesis class $\Hcal$ of $d$-layer neural networks $N_{W_1}^{d}$ 
	%	mapping	$\Xcal$ to $\reals$, such that for all $j$, $\norm{W_j}_F\leq 
	%	M_F(j)$ and $\sigma_j$ is applied element-wise, satisfies
	%	\[
	%%	\Ocal\left(\frac{B2^d\prod_{j=1}^{d}M_F(j)}{\sqrt{m}}\right)~.
	%	\]
	%\end{theorem}
	
	%\begin{corollary}\label{cor:neyshcol}
	%	If $\Hcal$ and $\ell$ satisfy the conditions of \thmref{thm:depthreduc} 
	%and 
	%	\thmref{thm:neyshabur15}, then
	%	\[
	%	\hat{\Rcal}_m(\ell\circ\Hcal)~\leq~ \tilde{\Ocal}\left(
	%	\frac{B\prod_{j=1}^{d}M_F(j)}{\gamma}\cdot \min\left\{
	%	
	%\frac{1}{\sqrt{m}}+\left(\frac{\log\left(\frac{1}{\Gamma}\prod_{j=1}^{d}M_p(j)\right)}
	%	{\log(m)}\right)^{1/p}~,~\frac{2^d}{\sqrt{m}}
	%	\right\}\right),
	%	\]
	%	where the $\tilde{\Ocal}$ notation hides doubly-logarithmic factors.
	%\end{corollary}
	
	In this section we exemplify how \thmref{thm:depthreduc} can be used to 
	obtain 
	depth-independent bounds on the sample complexity of various classes of 
	neural 
	networks. The general technique is as follows: First, we prove a  bound on 
	$\hat{R}_m(\Hcal_r)$, which generally depends on the depth $r$, and scales 
	as 
	$r^\alpha/\sqrt{m}$ for some $\alpha>0$. Then, we plug it into 
	\thmref{thm:depthreduc}, and utilize the following lemma to tune $r$ 
	appropriately:
	%\begin{lemma}
	%   \label{lem:ralphabeta}
	%For any $\alpha, \beta, b, n > 0$, $d \in \mathbb{N}$, it holds that
	%\[
	%\min \left\{ \min_{r \in \{1 ,\ldots, d\}} %\frac{r^\alpha}{n} + 
	%\frac{b}{r^\beta}, \frac{d^\alpha}{n} \right\} \leq \min %\left\{ 
	%2\frac{b^{\frac{\alpha}{\alpha + %\beta}}}{n^{\frac{\beta}{\alpha + 
	%%\beta}}}, 
	%\frac{d^\alpha}{n} \right\}.
	%\]
	%\end{lemma}
	
	\begin{lemma}
		\label{lem:ralphabeta}
		For any $\alpha>0$, $\beta\in (0,1]$ and $b,c,n\geq 1$, it holds that
		\[
		\min\left\{\min_{r\in\{1,\ldots,d\}}\frac{c 
			r^\alpha}{n}+\frac{b}{r^\beta}~,~\frac{d^\alpha}{n}\right\}
		~\leq 
		\min\left\{3\cdot\frac{b^{\frac{\alpha}{\alpha+\beta}}}{(n/c)^{\frac{\beta}{\alpha+\beta}}}~,~
		\frac{d^\alpha}{n}\right\}~.
		\]
	\end{lemma}
	
	We begin with proving a depth-independent version of \thmref{thm:sqrtl}. 
	That 
	theorem implies that for the class $\Hcal_r$ of depth-$r$ neural networks 
	with 
	Frobenius norm bounds $M_F(1),\ldots,M_F(r)$ (up to and including $r=d$),
	\begin{equation}\label{eq:radd23}
	\hat{R}_m(\Hcal_r)
	~\leq~
	\Ocal\left(B\prod_{j=1}^r 
	M_F(j)\sqrt{\frac{r}{m}}\right)
	\end{equation}
	Plugging this into \thmref{thm:depthreduc}, and using 
	\lemref{lem:ralphabeta}, 
	it is straightforward to derive the following corollary (see 
	\secref{sec:proofs} for a formal derivation):
	\begin{corollary}
		\label{cor:sqrtmrate}
		Let $\Hcal$ be the class of depth-$d$ neural networks, where each 
		parameter 
		matrix $W_j$ satisfies $\norm{W_j}_F\leq M_F(j)$, and with 
		$1$-Lipschitz, 
		positive-homogeneous, element-wise activation functions. Assuming the 
		loss 
		function $\ell$ and $\Hcal$ satisfy the conditions of 
		\thmref{thm:depthreduc} (with the sets $\Wcal_j$ being unconstrained),
		% and also that $B \prod_{j=1}^d M_F(j) \geq 1$,
		it 
		holds that
		\begin{align*}
		\hat \Rcal_m(\ell \circ \Hcal) \leq \Ocal \left(\frac{B \prod_{j=1}^d 
			M_F(j)}{\gamma} \cdot \min \left\{ 
			\frac{\bar{\log}^{3/4}(m)\sqrt{\bar{\log}\left( \frac 1 
				\Gamma 
				\prod_{j=1}^d M_F(j)\right)}}{m^{1/4}}~,~ 
		\sqrt{\frac{d}{m}}\right\}\right)~.
		\end{align*}
		%     $$
		%    \hat\Rcal_m(\ell \circ \Hcal) \leq \Ocal\left( 
		%%%\frac{\prod_{j=1}^d M_F(j) 
		%%\cdot \log^{p/(p+2)}(\frac B\Gamma \prod_{j=1}^d M_p(j)
		%) \cdot B \cdot \log^{3(p+2)}m}{\gamma m^{1/(p+2)}} %\right).
		%    $$
		where $\bar{\log}(z):=\max\{1,\log(z)\}$.
	\end{corollary}
	
	Ignoring logarithmic factors and replacing the $\min$ by its first 
	argument, 
	the bound in the corollary is at most
	\[
	\tilde{\Ocal}\left(\frac{B \prod_{j=1}^d 
		M_F(j)}{\gamma}\sqrt{\frac{\bar{\log}\left(\frac{1}{\Gamma}\prod_{j=1}^{d}
			M_F(j)\right)}{\sqrt{m}}}\right)~.
	\]
	Assuming $\prod_j M_F(j)$ and $\prod_j M_F(j)/\Gamma$ are bounded by a 
	constant, we get a bound which does not depend on the width or depth of the 
	network: In other words, it is possible to make this bound smaller than 
	any 
	fixed $\epsilon$, with a sample size $m$ independent of the network's size.
	On the other hand, the bound in Corollary \ref{cor:sqrtmrate} is also 
	bounded by
	\[
	\Ocal\left(\frac{B \prod_{j=1}^d 
		M_F(j)\sqrt{d}}{\gamma\sqrt{m}}\right),
	\]
	which is the bound one would get from an immediate application of 
	\thmref{thm:sqrtl}, and implies that the asymptotic rate (as a function of 
	$m$) 
	is still maintained. As discussed in the introduction, the assumption that 
	$\prod_j M_F(j)$ is a constant is certainly a strong one in practice, but 
	to the best of our knowledge, is the first norm-based assumption which 
	leads to size independence.

	Next, we apply \thmref{thm:depthreduc} to the results in 
	\citet{bartlett2017spectrally}, which as discussed in the introduction, 
	provide 
	a depth-dependent bound using a different set of norms. Specifically, they 
	obtain the following 
	intermediary result in deriving their generalization bound:
	\begin{theorem}[\citet{bartlett2017spectrally}]\label{thm:bartlett}
		Let $\Hcal$ be the hypothesis class of depth-$d$, width-$h$ real-valued 
		networks on $\Xcal = \{ \bx : \| \bx \| 
		\leq B \}$, using $1$-Lipschitz activation functions, given by
		\[
		\Hcal = \left\{ N_{W_1^d} : \forall j \in \{ 1, \ldots, d\}, 
		\norm{W_j}\leq 
		M(j),\frac{\norm{W_j^T}_{2,1}}{\norm{W_j}}\leq L(j) \right\},
		\]
		for some fixed parameters $\{L(j), M(j)\}_{j=1}^d$. Then the Rademacher 
		complexity $\hat \Rcal_m(\Hcal)$ is upper-bounded by
		\[
		\hat \Rcal_m (\Hcal) \leq \Ocal \left( \frac{B \log(h) \log(m) 
		\prod_{j=1}^d 
			M(j)}{\sqrt m} \cdot \left( \sum_{j=1}^d L(j)^{2/3}
		\right)^{3/2} \right).
		\]
	\end{theorem}
	As discussed in the introduction, $L(j)$ can never be smaller than $1$, 
	hence 
	the bound scales at least as $\sqrt{d^3/m}$. However, using the bound above 
	together with \thmref{thm:depthreduc} and \lemref{lem:ralphabeta}, we can 
	get 
	the following corollary (where for simplicity, we assume that 
	$L(j)$ for all $j$ are uniformly bounded by some $L$):
	\begin{corollary}
		\label{cor:sqrtmrate2}
		Let $\Hcal$ be the class of depth-$d$, width-$h$  networks with 
		$1$-Lipschitz, positive-homogeneous, element-wise activation functions.
		Assuming the loss function $\ell$ and $\Hcal$ satisfy the conditions of 
		\thmref{thm:depthreduc} (with 
		$\Wcal_j=\left\{W_j:\frac{\norm{W_j^T}_{2,1}}{\norm{W_j}}\leq 
		L\right\}$ 
		for all $j$), it holds that the Rademacher complexity $\hat 
		\Rcal_m(\ell 
		\circ \Hcal)$ is at most
		%\[
		%\Ocal \left( \frac{BL \log(h) \log(m) \prod_{j=1}^r M(j)}{\gamma}\cdot 
		%\min 
		%\left\{ \frac{\log \left( \frac{1}{\Gamma} \prod_{j=1}^d M_p(j) 
		%\right)^{3/8} 
		%\left(L \log^{3/2}(m)\right)^{1/4}}{m^{1/8}},\frac{ d^{3/2}}{\sqrt m} 
		%\right\} 
		%\right)~.
		%
		\[
		\Ocal \left( \frac{BL \log(h) \log(m) \prod_{j=1}^d M(j)}{\gamma}
		\cdot \min \left\{ \frac{\bar{\log} \left( \frac{1}{\Gamma} 
		\prod_{j=1}^d 
		M_p(j) 
			\right)^{\frac{1}{\frac{2}{3}+p}} \left( 
			\bar{\log}^{3/2}(m)\right)^{\frac{1}{1+\frac{3}{2}p}}}{m^{\frac{1}{2+3p}}},\frac{d^{3/2}}{\sqrt
			m} \right\} \right)~,
		\]
		where $\bar{\log}(z):=\max\{1,\log(z)\}$.		
	\end{corollary}
	As before, by replacing the $\min$ by its first argument, we get a bound 
	which is fully independent of the network size, assuming the norms are 
	suitably 
	bounded. To give a concrete example, if we take $p=2$ (so that the $M_p(j)$ 
	constraints correspond to the Frobenius norm), and ignore lower-order 
	logarithmic factors, we get a bound scaling as
	\[
	\hat \Rcal_m 
	(\ell\circ\Hcal)~\leq~\tilde{\Ocal}\left(\frac{BL\prod_{j=1}^{d}M(j)}{\gamma}\cdot\min\left\{
	\sqrt[4]{\frac{\log^{3/2} \left( \frac{1}{\Gamma} \prod_{j=1}^d M_p(j) 
			\right)}{m^{1/2}}}~,~\sqrt{\frac{d^{3}}{m}}\right\}\right).
	\]
	In contrast, a direct application of \thmref{thm:bartlett} in the same 
	setting 
	leads to a bound of
	\[
	\hat \Rcal_m (\ell\circ\Hcal)~\leq~\tilde{\Ocal}\left(
	\frac{BL\prod_{j=1}^{d}M(j)}{\gamma}\cdot\sqrt{\frac{d^3}{m}}\right)~.
	\]
	
	Finally, we note that since the \eqref{eq:neyshabur2017pac}, based on a 
	PAC-Bayes analysis, is always weaker than \eqref{eq:bartlett2017} (as noted 
	in \cite{bartlett2017spectrally}), \corollaryref{cor:sqrtmrate2} also gives 
	a size-independent version of \eqref{eq:neyshabur2017pac}.
	
	\if 0
	in a similar manner,
	\thmref{thm:depthreduc} is also potentially applicable to improving the 
	bound 
	of \citet{neyshabur2017pac} (see \eqref{eq:neyshabur2017pac}). However, 
	their 
	result bounds the generalization error directly, using PAC-Bayesian 
	techniques, 
	and does not directly imply a bound on the Rademacher complexity; moreover, 
	as noted by \cite{bartlett2017spectrally}, . Therefore, it 
	is unclear whether a depth-independent version of 
	\eqref{eq:neyshabur2017pac} 
	is possible (although we conjecture this is the case). 
	\fi

	\section{A Lower Bound for Schatten Norms}\label{sec:lowerbound}
	
	In this section, we present a lower bound on the Rademacher complexity, for 
	the 
	class of neural networks with parameter matrices of bounded Schatten norms. 
	The 
	formal result is the following:
	\begin{theorem}\label{thm:lowerbound}
		Let $\Hcal$ be the class of depth-$d$, width-$h$ neural networks, where 
		each 
		parameter matrix $W_j$ satisfies $\norm{W_j}_p\leq M_p(j)$ for some 
		Schatten $p$-norm $\norm{\cdot}_p$ (and where we use the convention 
		that 
		$p=\infty$ refers to the spectral norm). Then there exists a choice of 
		$\frac{1}{\gamma}$-Lipschitz loss $\ell$ and data points 
		$\bx_1,\ldots,\bx_m\in\Xcal$, with respect to which
		\[
		\hat{\Rcal}_m(\ell\circ\Hcal)~\geq~    \Omega\left(
		\frac{B\prod_{j=1}^{d}M_p(j)\cdot 
			h^{\max\left\{0,\frac{1}{2}-\frac{1}{p}\right\}}}{\gamma\sqrt{m}}\right).
		\]
	\end{theorem}
	
	This theorem strengthens Theorem 3.6 in \citet{bartlett2017spectrally}, 
	in the sense that they only considered the case $p=\infty$, and did not have a dependence on $h$. On the other hand, they consider bounds which hold for any choice of $\bx_1,\ldots,\bx_m$, while we consider bounds uniform over 
	$\bx_1,\ldots,\bx_m$ for simplicity. Moreover, for network depths larger than $1$, our construction requires a non element-wise activation function. The lower bound has the following 
	implications:
	\begin{itemize}
		\item Like \citet{bartlett2017spectrally}, the theorem implies that by 
		controlling just the norms of each parameter matrix, a dependence on 
		the 
		product of the norms is generally inevitable.
		\item For $p=\infty$, we see that there is an inevitable $h^{1/2}$ 
		factor 
		in the bound, which implies that controlling the spectral norm is 
		insufficient to get size-independent bounds (at least, independent of 
		the 
		width $h$). More generally, any Schatten $p$-norm control with $p>2$ 
		will 
		be insufficient to get such size independence. 
		\item For $p=2$ (i.e. Frobenius norm bounds $M_F(1),\ldots,M_F(d)$), 
		the 
		lower bound becomes size-independent, and on the order of
		$\frac{B\prod_{j=1}^{d}M_F(j)}{ \sqrt{m}}$. The dependence on $M_F(j)$ is similar to our upper bound in Corollary \ref{cor:sqrtmrate} up to logarithmic factors (although the dependence on $m$ is worse).
	\end{itemize}

	\section{Additional Remarks}\label{sec:remarks}

	\subsection{Post-hoc Guarantees}
	
	So far we have proved upper bounds on the empirical Rademacher complexity 
	of a 
	fixed class of neural networks, of the form
	$$\Hcal_L = \{ h: C(h)\leq L\},$$ for some complexity measure 
	$C(\cdot)$
	and a parameter $L$. These imply high-probability learning guarantees for  
	algorithms which return predictors in $\Hcal_L$. However, in the context of 
	norm-based constraints, practical algorithms for neural networks usually 
	perform unconstrained optimization, and therefore are not guaranteed 
	a-priori 
	to return a predictor in some fixed $\Hcal_L$. Fortunately, it is 
	straightforward to convert these bounds into probabilistic guarantees for 
	\emph{any} neural network $h$, with the bound scaling appropriately with 
	the 
	complexity $C(h)$ of the particular network. We note that such 
	post-hoc guarantees have also been stated in the context of some previous 
	sample complexity bounds for neural networks (e.g., 
	\citet{bartlett2017spectrally,neyshabur2017pac}). It is achieved, for 
	instance, 
	by a union bound over, say, a doubling scale of the complexity. We refer to 
	the 
	proof of the margin bound of \citet[Theorem 2]{koltchinskii2002empirical} 
	for 
	an example of this technique.

	\subsection{Complexity of Lipschitz Networks}
	
	In proving the results of \secref{sec:depthind}, the key element has been 
	the 
	observation that under appropriate norm constraints, a neural network must 
	have 
	a layer with parameter matrix close to being rank-1, and therefore the 
	network 
	can be viewed as a composition of a shallower network and a univariate 
	Lipschitz function. In fact, this can be generalized: Whenever we have a 
	network with parameter matrix close to being rank-$k$, we can view it as a 
	composition of a shallow network and a Lipschitz function on $\reals^k$. 
	Although we do not develop this idea further in this paper, this 
	observation 
	might be useful in analyzing other types of neural network classes.
	
	Taking this to the extreme, we can also bound the complexity of neural 
	networks 
	computing Lipschitz functions, by studying the complexity of all Lipschitz 
	functions over the domain $\Xcal$. It is easily verified that in our 
	setting, 
	if we consider the class $\Hcal$ of depth-$d$ networks, where each 
	parameter 
	matrix has spectral norm at most $M(j)$, then the network must be 
	$\prod_{j=1}^{d}M(j)$-Lipschitz. Using well-known estimates of the covering 
	numbers of Lipschitz functions over $\Xcal$, we get that
	\[
	\hat{\Rcal}_m(\ell\circ\Hcal)~\leq~ 
	\Ocal\left(\frac{B\prod_{j=1}^{d}M(j)}{\gamma m^{1/\text{dim}(\Xcal)}}\cdot 
	\right)~,
	\]
	where the loss $\ell$ is assumed $\frac{1}{\gamma}$-Lipschitz, and where 
	$\text{dim}(\Xcal)$ is the dimensionality of $\Xcal$. Of course, the bound 
	has 
	a very bad dependence on the input dimension (or equivalently, the width of 
	the 
	first layer in the network), but on the other hand, has no dependence on 
	the 
	network depth, nor on any matrix norm other than the spectral norm. Again, 
	as 
	discussed in the previous subsection, it is also possible to use this bound 
	to 
	get post-hoc guarantees, without constraining the Lipschitz parameter of 
	the 
	learned network in advance.

	\section{Proofs}\label{sec:proofs}
	
	\subsection{Proofs of \lemref{lem:contraction_l1} and \thmref{thm:sqrtl1}}
	
	We first prove \lemref{lem:contraction_l1}. 
	Letting $\bw_j$ denote the $j$-th row of a matrix $W$, we have
	\begin{align*}
	\sup_{f\in\Fcal,W:\norm{W}_{1,\infty}\leq R}g\left(\left\|\sum_{i=1}^{m}
	\epsilon_i\sigma(Wf(\bx_i))\right\|_\infty\right) 
	&=\sup_{f\in\Fcal,W:\norm{\bw_j}_1\leq R} \max_{k} g\left( 
	\left|\sum_{i=1}^{m}
	\epsilon_i\sigma(\bw_k^\top f(\bx_i))\right| \right) \\
	&=\sup_{f\in\Fcal,\norm{\bw}_1\leq R} g\left( \left|\sum_{i=1}^{m}
	\epsilon_i\sigma(\bw^\top f(\bx_i))\right| \right) .
	\end{align*}
	Since $g(|z|)\leq g(z)+g(-z)$, the left-hand side of 
	\eqref{eq:contr_lemma2_statement} is upper 
	bounded by
	\begin{align*}
	\E_{\boldsymbol{\epsilon}}\sup_{f\in\Fcal,\norm{\bw}_1\leq R} g\left( 
	\sum_{i=1}^{m}
	\epsilon_i\sigma(\bw^\top f(\bx_i))\right) + \E_{\boldsymbol{\epsilon}} 
	\sup_{f\in\Fcal,\norm{\bw}_1\leq R} g\left( -\sum_{i=1}^{m}
	\epsilon_i\sigma(\bw^\top f(\bx_i))  \right)
	\end{align*}
	and the proof is concluded exactly as in Lemma~\ref{lem:contraction} by 
	appealing to \citet{ledoux2013probability}.
	
	We now turn to \thmref{thm:sqrtl1}, whose proof is rather similar to that 
	of 
	\thmref{thm:sqrtl}. Fixing $\lambda>0$ to be chosen later, the Rademacher 
	complexity can be upper bounded as
	\begin{align*}
	m\hat{\Rcal}_m(\Hcal_d) &= \E_{\boldsymbol{\epsilon}} 
	\sup_{N_{W_1^{d-1}},W_d} \sum_{i=1}^m 
	\epsilon_i W_d \sigma_{d-1}(N_{W_1^{d-1}}(\bx_i))  \\
	&\leq \frac{1}{\lambda} \log \E_{\boldsymbol{\epsilon}} \sup \exp 
	\lambda \sum_{i=1}^m 
	\epsilon_i W_d \sigma_{d-1}(N_{W_1^{d-1}}(\bx_i)) \\
	&\leq \frac{1}{\lambda} \log \E_{\boldsymbol{\epsilon}} \sup \exp  
	\left\{ M(d)\cdot 
	\left\|\lambda \sum_{i=1}^m \epsilon_i 
	\sigma_{d-1}(N_{W_1^{d-1}}(\bx_i)) \right\|_{\infty} \right\}
	\end{align*}
	Applying the same argument as in the proof of \thmref{thm:sqrtl}, and using 
	\lemref{lem:contraction_l1}, we can upper bound the above by
	\begin{equation}\label{eq:radboundl1}
	m\hat{\Rcal}_m(\Hcal_d) \leq \frac{1}{\lambda} \log \left(2^d \cdot 
	\E_{\boldsymbol{\epsilon}} 
	\exp \left( M \lambda \left\| \sum_{i=1}^m \epsilon_i \bx_i 
	\right\|_{\infty} \right) \right)~,
	\end{equation}
	where $M= \prod_{j=1}^d M(j)$. Letting $x_{i,j}$ denote the $j$-th 
	coordinate of $\bx_i$, and using symmetry, the expectation inside the log 
	can be re-written as
	\begin{align*}
	\E_{\boldsymbol{\epsilon}}\exp &\left( M \lambda\cdot \max_j \left| 
	\sum_{i=1}^m \epsilon_i 
	x_{i,j} 
	\right| \right)
	~\leq~ \sum_{j=1}^{n} \E_{\boldsymbol{\epsilon}}\exp\left( M \lambda\cdot 
	\left| \sum_{i=1}^m 
	\epsilon_i 
	x_{i,j} 
	\right| \right)\\
	&\leq~\sum_{j=1}^{n} \E_{\boldsymbol{\epsilon}}\left[\exp\left( M \lambda 
	\sum_{i=1}^m \epsilon_i 
	x_{i,j} 
	\right)+\exp\left(- M \lambda\sum_{i=1}^m \epsilon_i 
	x_{i,j} 
	\right)\right]\\
	&=~2\sum_{j=1}^{n} 
	\E_{\boldsymbol{\epsilon}}\exp\left(M\lambda\sum_{i=1}^{m}\epsilon_i 
	x_{i,j}\right)
	~=~2\sum_{j=1}^{n}\prod_{i=1}^{m}\E_{\boldsymbol{\epsilon}}\exp\left(M\lambda\epsilon_i
	x_{i,j}\right)\\
	&=~2\sum_{j=1}^{n}\prod_{i=1}^{m}\frac{\exp\left(M\lambda 
		x_{i,j}\right)+\exp\left(-M\lambda x_{i,j}\right)}{2}
	~\leq~ 2\sum_{j=1}^{n} 
	\exp\left(M^2\lambda^2\sum_{i=1}^{m}x_{i,j}^2\right)~,
	\end{align*}
	where in the last step we used the fact that 
	$\frac{\exp(z)+\exp(-z)}{2}\leq \exp(z^2/2)$. Further upper bounding this 
	by $2n\max_j \exp\left(M^2\lambda^2\sum_{i=1}^{d}x_{i,j}^2\right)$ and 
	plugging back to \eqref{eq:radboundl1}, we get 
	\begin{align*}
	\frac{1}{\lambda} \log \left(2^{d+1} n\cdot \max_j 
	\exp\left(M^2\lambda^2\sum_{i=1}^{m}x_{i,j}^2\right)\right)
	~=~
	\frac{d+1+\log(n)}{\lambda}+M^2\lambda\max_j \sum_{i=1}^{m}x_{i,j}^2~.
	\end{align*}
	Choosing $\lambda = 
	\sqrt{\frac{d+1+\log(n)}{M^2\max_j\sum_{i=1}^{m}x_{i,j}^2}}$, we can upper 
	bound the above by
	\[
	2M\sqrt{\left(d+1+\log(n)\right)\max_j\sum_{i=1}^{m}x_{i,j}^2},
	\]
	from which the result follows.

	\subsection{Proof of \thmref{thm:rankonereplace}}
	
	The proof will build on the following few technical lemmas.
	
	\begin{lemma}\label{lem:rankoneapprox}
		For any matrix $W$, and any Schatten $p$-norm $\norm{\cdot}_p$ such 
		that 
		$p<\infty$, there exists a 
		rank-1 matrix $\tilde{W}$ of the same size such that
		\[
		\norm{\tilde{W}}\leq 
		\norm{W}~~,~~\norm{\tilde{W}}_p\leq\norm{W}_p~~,~~\norm{W-\tilde{W}} 
		\leq 
		\left(\norm{W}^p_p-\norm{W}^p\right)^{1/p}.
		\]
	\end{lemma}
	\begin{proof}
		Let $USV^\top$ denote the singular value decomposition of $W$, where 
		$S=\text{diag}(s_1,s_2,\ldots,s_r)$, and choose $\tilde{W}=\bu_1 
		s_{1}\bv_1^\top$, 
		where $\bu_1,\bv_1,s_1$ are top singular vectors and values of $W$. The 
		first two inequalities in the lemma are easy to verify. As to the third 
		inequality, using 
		the unitarial invariance of the spectral norm, we have
		\[
		\norm{W-\tilde{W}} = \norm{USV^\top-\bu_1s_1\bv_1^\top} = 
		\norm{U\text{diag}(0,s_2,\ldots,s_r)V^\top} = s_2 ~\leq~ 
		\left(\sum_{j=2}^{r}s_j^p\right)^{1/p}~=~ 
		\left(\sum_{j=1}^{r}s_j^p-s_1^p\right)^{1/p},
		\]
		which equals $\left(\norm{W}_p^p-\norm{W}^p\right)^{1/p}$.
	\end{proof}
	
	\begin{lemma}\label{lem:perturb}
		Given network parameters $W_1^d=\{W_1,\ldots,W_d\}$, let 
		$\tilde{W}_1^d=\{W_1,\ldots,W_{s-1},\tilde{W}_s,W_{s+1},\ldots,W_d\}$ 
		be 
		the same parameters, where the parameter matrix $W_s$ of the $s$-th 
		layer 
		(for some fixed $s\in\{1,\ldots,d\}$) is changed to some other matrix 
		$\tilde{W}_s$. Then 
		\[
		\sup_{\bx\in\Xcal}\norm{N_{W_1^d}(\bx)-N_{\tilde{W}_1^d}(\bx)} 
		~\leq~ 
		B\left(\prod_{j=1}^{d}\norm{W_{j}}\right)\frac{\norm{W_s-\tilde{W}_s}}{\norm{W_s}}~.
		\]
	\end{lemma}
	\begin{proof}
		By a simple calculation, we have that the Lipschitz constant of the 
		function
		$N_{W_b^s}$ is at most $\prod_{j=b}^{s}\norm{W_j}$. 
		
		Assume for now that $2\leq s\leq d-1$. By definition, we have 
		\[
		N_{W_1^d}(\bx) = 
		N_{{W}_{s+1}^d}(\sigma_{s}(W_s\sigma_{s-1}(N_{W_1^{s-1}}(\bx))))
		\]
		and 
		\[N_{\tilde{W}_1^d}(\bx) = 
		N_{W_{s+1}^d}(\sigma_{s}(\tilde{W}_s\sigma_{s-1}(N_{W_1^{s-1}}(\bx))))~.
		\]
		The Lipschitz constant of the 
		function $N_{W_{s+1}^{d}}$ is at most $\prod_{j=s+1}^{d}\norm{W_j}$, 
		and 
		the norm 
		of $N_{W_1^{s-1}}(\bx)$ is at most $\norm{\bx}\prod_{j=1}^{s-1}\norm{W_j}$.
		Therefore, for any $\bx\in\Xcal$,
		\begin{align*}
		\norm{N_{W_1^d}(\bx)-N_{\tilde{W}_1^d}(\bx)} &~=~
		\norm{N_{W_{s+1}^d}(\sigma_{s}(W_s\sigma_{s-1}(N_{W_1^{s-1}}(\bx))))
			-
			N_{W_{s+1}^d}(\sigma_{s}(\tilde{W}_s\sigma_{s-1}(N_{W_1^{s-1}}(\bx))))\
		}\\
		&~\leq~\left(\prod_{j=s+1}^{d}\norm{W_j}\right)\cdot
		\norm{\sigma_s(W_s\sigma_{s-1}(N_{W_1^{s-1}}(\bx)))-
			\sigma_{s}(\tilde{W}_s\sigma_{s-1}(N_{W_1^{s-1}}(\bx)))}\\
		&~\leq~ \left(\prod_{j=s+1}^{d}\norm{W_j}\right)\cdot
		\norm{W_s\sigma_{s-1}(N_{W_1^{s-1}}(\bx))-
			\tilde{W}_s\sigma_{s-1}(N_{W_1^{s-1}}(\bx)))}\\
		&~\leq~\left(\prod_{j=s+1}^{d}\norm{W_j}\right)\cdot
		\norm{W_s-\tilde{W}_s}\cdot\norm{\sigma_{s-1}(N_{W_1^{s-1}}(\bx)}\\
		&~\leq~\left(\prod_{j=s+1}^{d}\norm{W_j}\right)\cdot\norm{W_s-\tilde{W}_s}
		\cdot\left(\prod_{j=1}^{s-1}\norm{W_j}\right)\cdot\norm{\bx},
		\end{align*}
		from which the result follows after a simplification. The cases $s=1$ 
		and 
		$s=d$ are handled in exactly the same manner. 
	\end{proof}
	
	\begin{lemma}\label{lem:closetorankone}
		Suppose that $N_{W_1^d}$ is such that 
		$\prod_{j=1}^{d}\norm{W_j}\geq \Gamma$ and 
		$\prod_{j=1}^{d}\norm{W_j}_{p}\leq M$. Then for 
		any 
		$r\in \{1,\ldots,d\}$,
		\[
		\min_{j\in \{1,\ldots,r\}}\frac{\norm{W_j}_p}{\norm{W_j}} ~\leq~ 
		\left(\frac{M}{\Gamma}\right)^{1/r}.
		\]
	\end{lemma}
	\begin{proof}
		Fixing some $r$, and using the stated assumptions as well as the fact 
		that 
		$\norm{W}_p\geq \norm{W}$ for any $p$, we have
		\[
		\frac{M}{\Gamma}~\geq~ 
		\frac{\prod_{j=1}^{d}\norm{W_j}_{p}}{\prod_{j=1}^{d}\norm{W_j}}
		~=~ \prod_{j=1}^{d}\frac{\norm{W_j}_p}{\norm{W_j}}
		~\geq~ \prod_{j=1}^{r}\frac{\norm{W_j}_p}{\norm{W_j}}
		~\geq~ 
		\left(\min_{j\in\{1,\ldots,r\}}\frac{\norm{W_j}_p}{\norm{W_j}}\right)^{r}.
		\]
		Taking the $r$-th root from both sides, the result follows.
	\end{proof}
	
	With these lemmas in hand, we can now turn to prove 
	\thmref{thm:rankonereplace}. A direct application of 
	\lemref{lem:closetorankone} implies that for any $r\in \{1,\ldots,d\}$,
	\begin{equation}
	\label{eq:jstar}
	\min_{j\in \{1,\ldots,r\}}\frac{\norm{W_j}_p}{\norm{W_j}} ~\leq~ 
	\left(\frac{M}{\Gamma}\right)^{1/r}~.
	\end{equation}	
	Let $r' \in \{1, \ldots, r\}$ be the value of $j$ for which the above minimum is obtained. Combining \lemref{lem:rankoneapprox} and 
	\lemref{lem:perturb} (where the value of $s$ in \lemref{lem:rankoneapprox} is set to $r'$), we have that there 
	indeed exists a network $N_{\tilde{W}_1^d}$ with a rank-1 matrix in layer 
	$r'$, such that
	\begin{align}
	\sup_{\bx\in\Xcal}\norm{N_{W_1^d}(\bx)-N_{\tilde{W}_1^d}(\bx)}
	&~\leq~ 
	B\left(\prod_{j=1}^{d}\norm{W_{j}}\right)\cdot
	\frac{\left(\norm{W_{r'}}^p_p-\norm{W_{r'}}^p\right)^{1/p}}{\norm{W_{r'}}}\notag\\
	&~=~B\left(\prod_{j=1}^{d}\norm{W_{j}}\right)
	\left(\frac{\norm{W_{r'}}^p_p}{\norm{W_{r'}}^p}-1\right)^{1/p}.\label{eq:bound1}
	\end{align}

	Substituting \eqref{eq:jstar} into \eqref{eq:bound1}, we get that 
	\begin{align*}
	\sup_{\bx\in\Xcal}\norm{N_{W_1^d}(\bx)-N_{\tilde{W}_1^d}(\bx)}
	&~\leq~
	B\left(\prod_{j=1}^{d}\norm{W_{j}}\right)
	\left(\left(\frac{M}{\Gamma}\right)^{p/r}-1\right)^{1/p}\\
	&~=~
	B\left(\prod_{j=1}^{d}\norm{W_{j}}\right)
	\left(\exp\left(\frac{p\log(M/\Gamma)}{r}\right)-1\right)^{1/p}~.
	\end{align*}
	Suppose for now that $r$ is such that $r\geq p\log(M/\Gamma)$. 
	Using the fact that $\exp(z)\leq 1+2z$ for any $z\in [0,1]$, it follows 
	that the above is at most
	\[
	B\left(\prod_{j=1}^{d}\norm{W_j}\right)\left(\frac{2p\log(M/\Gamma)}{r}\right)^{1/p}.
	\]
	It remains to consider the case where $r< p\log(M/\Gamma)$. However, in 
	this regime the theorem trivially holds: let $r'=r$, and let $\tilde{W}_{r}$ in the network 
	$N_{\tilde{W}_1^{r}}$ be the all-zeros matrix (which is rank zero and ensures 
	that $N_{\tilde{W}_1^{r}}(\bx)=\mathbf{0}$ for all $\bx$), and we have by 
	definition that
	\[
	\sup_{\bx\in\Xcal}\norm{N_{W_1^d}(\bx)-N_{\tilde{W}_1^d}(\bx)}=
	\sup_{\bx :\norm{\bx}\leq 
		B}\norm{N_{W_1^d}(\bx)}\leq
	B\prod_{j=1}^{d}\norm{W_j}
	<
	B\left(\prod_{j=1}^{d}\norm{W_j}\right)\left(\frac{2p\log(M/\Gamma)}{r}\right)^{1/p}.
	\]
	
	\subsection{Proof of \thmref{thm:radlipschitz}}
	
	To prove the theorem, we use a straightforward covering number argument, 
	beginning with a few definitions.
	
	Given any function class $\Fcal$, a metric $d$ on the elements of $\Fcal$, 
	and 
	$\epsilon>0$, we 
	let the covering number $\Ncal(\Fcal,d,\epsilon)$ denote the minimal number 
	$n$ 
	of functions $f_1,f_2,\ldots,f_n$ in $\Fcal$, such that for all 
	$f\in\Fcal$, 
	$\min_{i=1,\ldots,n} d(f_i,f)\leq \epsilon$. In particular, fix some set of 
	data points $\bx_1,\ldots,\bx_m$, and define the empirical $L_2$ 
	distance
	\[
	\hat{d}_m(f,f') = \sqrt{\frac{1}{m}\sum_{i=1}^{m}(f(\bx_i)-f'(\bx_i))^2}~.
	\]
	Also, given a function class $\Fcal$ and $\bx_1,\ldots,\bx_m$, we let 
	\[
	\hat{\Gcal}_m(\Fcal):=\E_{\eta}\left[\sup_{f\in\Fcal}
	\frac{1}{m}\sum_{i=1}^{m} \eta_i f(\bx_i)\right]
	\]
	denote the (empirical) Gaussian complexity of $\Fcal$, where 
	$\eta_1,\ldots,\eta_n$ are i.i.d. standard Gaussian random 
	variables. It is well known that 
	$\hat{\Rcal}_m(\Hcal)$ and $\hat{\Gcal}_m(\Hcal)$ are equivalent up to a 
	$c\sqrt{\log(m)}$ factor \citep[pg. 97]{ledoux2013probability}. By 
	Sudakov's 
	minoration theorem (see Theorem 3.18 in 
	\citet{ledoux2013probability}), we have that for all $\alpha>0$
	\begin{equation}\label{eq:covnumbound1}
	\log(\Ncal(\Hcal,\hat{d}_m,\alpha))\leq 
	c\left(\frac{\sqrt{m}\cdot\hat{\Gcal}_m(\Hcal)}{\alpha}\right)^2
	\end{equation}
	for a universal constant $c>0$. 
	
	With these definitions in hand, we turn to prove the theorem. We first note 
	that $\Fcal_{L,a}\circ\Hcal$ is equivalent to the class 
	$\{Lg(\cdot)+a:g\in\Fcal_{1,0}\circ\Hcal\}$, and therefore
	\begin{align}\label{eq:covnumbound0}
	\hat{\Rcal}_m(\Fcal_{L,a}\circ\Hcal)&= 
	\E_{\boldsymbol\epsilon}\left[\sup_{g\in 
		\Fcal_{L,a}\circ\Hcal}\frac{1}{m}\sum_{i=1}^{m}\epsilon_i 
		g(\bx_i)\right]
	~=~
	\E_{\boldsymbol\epsilon}\left[\sup_{g\in 
		\Fcal_{1,0}\circ\Hcal}\frac{1}{m}\sum_{i=1}^{m}\epsilon_i 
	(Lg(\bx_i)+a)\right]\notag\\
	&=\E_{\boldsymbol\epsilon}\left[L\cdot \sup_{g\in 
		\Fcal_{1,0}\circ\Hcal}\left(\frac{1}{m}\sum_{i=1}^{m}\epsilon_i 
	g(\bx_i)\right)+\frac{a}{m}\sum_{i=1}^{m}\epsilon_i\right]\notag\\
	&=L\cdot \E_{\boldsymbol\epsilon}\left[\sup_{g\in 
		\Fcal_{1,0}\circ\Hcal}\frac{1}{m}\sum_{i=1}^{m}\epsilon_i 
	g(\bx_i)\right]
	~=~L\cdot\hat{\Rcal}_m(\Fcal_{1,0}\circ\Hcal).
	\end{align}
	Therefore, it is enough to consider  
	$\hat{\Rcal}_m(\Fcal_{1,0}\circ\Hcal)$. To simplify notation in what 
	follows, 
	we will drop the $1,0$ subscript from $\Fcal$. 
	
	We first argue that
	\begin{equation}\label{eq:covnumbound2}
	\log(\Ncal(\Fcal,\hat{d}_m,\epsilon))~\leq~\log(\Ncal(\Fcal,\hat{d}_\infty,\epsilon)) \leq~
	c'\left(1+\frac{R}{\epsilon}\right),
	\end{equation}
	again for some numerical constant $c'>0$. 
	The first inequality follows from the fact that for any functions $f,f'$, it holds that 
	$\hat{d}_m(f,f')\leq 
	\hat{d}_{\infty}(f,f'):=\sup_{\bx}|f(\bx)-f'(\bx')|$, so it 
	is 
	enough to upper bound 
	$\Ncal(\Fcal,\hat{d}_{\infty},\epsilon)$. To do so, we first notice that 
	the range of any $f\in\Fcal$ is in $[-R,R]$. Discretize 
	$[-R,R]\times[-R,R]$ into a two-dimensional grid $\mathcal{U}_x\times 
	\mathcal{U}_y$, where
	\[
	\mathcal{U}_x := 
	\left\{-R,-R+\epsilon,-R+2\epsilon,\ldots,-R+\left\lfloor 
	2R/\epsilon\right\rfloor\epsilon,R\right\}~~~,~~~\mathcal{U}_y := 
	\{0,\pm\epsilon,\pm 2\epsilon,\ldots,\pm \lfloor 
	R/\epsilon\rfloor\epsilon,\pm 
	R\}.
	\]
	Given any $f\in\Fcal$, construct the piecewise-linear function $f'$ as 
	follows: For any input $x\in \mathcal{U}_x$, let $f'(x)$ be the point in 
	$\mathcal{U}_y$ nearest to $f(x)$ (breaking ties arbitrarily), and let the 
	rest 
	of $f'$ be constructed as a linear interpolation of these points on 
	$\mathcal{U}_x$. It is easily verified that $\sup_{x\in 
	[-R,R]}|f(x)-f'(x)|\leq 
	\epsilon$. Moreover, note that for two neighboring points $x,x'$ in 
	$\mathcal{U}_x$, the points $f'(x),f'(x')$ on $\mathcal{U}_y$ must be 
	neighboring or 
	the same. Therefore, each such function $f'$ can be parameterized by a 
	vector 
	of the form $\{-,0,+\}^{|\mathcal{U}_x|-1}$, which specifies whether 
	(starting 
	from the origin) $f'$ goes up, down, or remains the same on each of its 
	linear 
	segments. The number of such functions is at most 
	$3^{|\mathcal{U}_x|-1}\leq 
	3^{2R/\epsilon+1}$, and therefore 
	$\Ncal(\Fcal,\hat{d}_{\infty},\epsilon)\leq 
	3^{2R/\epsilon+1}$. Recalling that  
	$\hat{d}_\infty(f,f')$ majorizes $\hat{d}_{m}(f,f')$, we get 
	\eqref{eq:covnumbound2}.
	
	Next, we argue that 
	\begin{equation}\label{eq:covnumbound3}
	\Ncal(\Fcal\circ \Hcal,\hat{d}_m,\epsilon) ~\leq~ 
	\Ncal\left(\Fcal,\hat{d}_\infty,\frac{\epsilon}{2}\right)\cdot 
	\Ncal\left(\Hcal,\hat{d}_m,\frac{\epsilon}{2}\right)~.
	\end{equation}
	To see this, pick any $f\in \Fcal$ and $h\in 
	\Hcal$, and let $f',h'$ be the respective closest functions in the cover of 
	$\Fcal$ and $\Hcal$ (at scale $\epsilon/2$). By the triangle inequality 
	and 
	the easily verified fact that $f'$ is $1$-Lipschitz, we have 
	\begin{align*}
	\hat{d}_m(fh,f'h')&\leq \hat{d}_m(fh,f'h)+\hat{d}_m(f'h,f'h')
	~\leq~\hat{d}_{\infty}(fh,f'h)+\hat{d}_m(f'h,f'h')\\
	&\leq
	\sup_{x\in[-R,R]}|f(x)-f'(x)|+\sqrt{\frac{1}{m}\sum_{i=1}^{m}\left(f'h(\bx_i)-f'h'(\bx_i)\right)^2}\\
	&\leq 
	\frac{\epsilon}{2}+\sqrt{\frac{1}{m}\sum_{i=1}^{m}\left(h(\bx_i)-h'(\bx_i)\right)^2}\\
	&=\frac{\epsilon}{2}+\hat{d}_m(h,h')~\leq~ 
	\frac{\epsilon}{2}+\frac{\epsilon}{2}~=~\epsilon~.
	\end{align*}
	Therefore, we can cover $\Fcal\circ\Hcal$ (at scale $\epsilon$) by 
	taking 
	$f'(h'(\cdot))$ for all possible choices of $f',h'$ from the covers of 
	$\Fcal,\Hcal$ (at scale $\epsilon/2$), leading to 
	\eqref{eq:covnumbound3}. 
	
	Combining \eqref{eq:covnumbound1}, \eqref{eq:covnumbound2} and 
	\eqref{eq:covnumbound3}, we get that
	\[
	\log(\Ncal(\Fcal\circ \Hcal,\hat{d}_m,\epsilon))~\leq~ 
	c''\left(1+\frac{R}{\epsilon}+\left(\frac{\sqrt{m}\cdot\hat{\Gcal}_m(\Hcal)}{\epsilon}\right)^2\right)~,
	\]
	for some numerical constant $c''>0$. Finally, we use Dudley's entropy 
	integral, 
	which together with the equation above, implies the following for some 
	numerical constant $c>0$ (possibly changing from row to row):
	\begin{align*}
	\hat{\Rcal}_m(\Fcal\circ \Hcal)&\leq c\inf_{\alpha\geq 
		0}\left\{\alpha+\frac{1}{\sqrt{m}}\int_{\alpha}^{\sup_{g\in 
			\Fcal\circ\Hcal}\hat{d}_m(g,0)}\sqrt{\log(\Ncal(\Fcal\circ 
		\Hcal,\hat{d}_m,\epsilon))}d\epsilon\right\}\\
	&\leq
	c\inf_{\alpha\geq 
		0}\left\{\alpha+\frac{1}{\sqrt{m}}\int_{\alpha}^{R}\sqrt{\log(\Ncal(\Fcal\circ
		\Hcal,\hat{d}_m,\epsilon))}d\epsilon\right\}\\
	&\leq
	c\inf_{\alpha\geq 0}
	\left\{\alpha+\frac{1}{\sqrt{m}}\int_{\alpha}^{R}\sqrt{1+\frac{R}{\epsilon}
		+\left(\frac{\sqrt{m}\cdot\hat{\Gcal}_m(\Hcal)}{\epsilon}\right)^2}d\epsilon\right\}\\
	&\leq
	c\inf_{\alpha\geq 0}
	\left\{\alpha+\frac{1}{\sqrt{m}}\int_{\alpha}^{R}\left(1+\sqrt{\frac{R}{\epsilon}}
	+\left|\frac{\sqrt{m}\cdot\hat{\Gcal}_m(\Hcal)}{\epsilon}\right|\right)d\epsilon\right\}\\
	&\leq
	c\inf_{\alpha\geq 0}
	\left\{\alpha+\frac{1}{\sqrt{m}}\int_{0}^{R}\left(1+\sqrt{\frac{R}{\epsilon}}\right)d\epsilon
	+\hat{\Gcal}_m(\Hcal)\int_{\alpha}^{R}\frac{1}{\epsilon}d\epsilon\right\}\\
	&\leq 
	c\inf_{\alpha\geq 0}
	\left\{\alpha+\frac{R}{\sqrt{m}}
	+\hat{\Gcal}_m(\Hcal)\log\left(\frac{R}{\alpha}\right)\right\}~.
	\end{align*}
	Choosing in particular $\alpha=R/\sqrt{m}$, we get the upper bound
	\[
	\hat{\Rcal}_m(\Fcal\circ 
	\Hcal)~\leq~c\left(\frac{R}{\sqrt{m}}+\log(m)\cdot\hat{\Gcal}_m(\Hcal)\right)
	\]
	for some $c>0$. Plugging this into \eqref{eq:covnumbound0}, and upper 
	bounding 
	$\hat{\Gcal}_m(\Hcal)$ by $c\sqrt{\log(m)}\hat{\Rcal}_m(\Hcal)$ (see 
	\citep{ledoux2013probability}), the result follows. 
	
	\subsection{Proof of \thmref{thm:depthreduc}}
	
	We first need the following lemma, which bounds the empirical Rademacher complexity of the union of a collection of bounded function classes.
	\begin{lemma}
	\label{lem:radunion}
	 Fix $\bx_1, \ldots, \bx_m$. Suppose $\Hcal_1, \ldots, \Hcal_r$ are classes of functions uniformly bounded by some $A \in \reals_+$ on $\bx_1, \ldots, \bx_m$. Then
	 $$
	 \hat{\Rcal}_m(\Hcal_1 \cup \cdots \cup \Hcal_r) \leq \max_{1 \leq i \leq r} \hat{\Rcal}_m(\Hcal_i) +\Ocal\left(A \sqrt{\frac{ \log r}{m}}\right).
	 $$
	\end{lemma}
	\begin{proof}
	The inequality is trivial if $r = 1$, so we may assume without loss of generality that $r \geq 2$. Fix some $j, 1 \leq j \leq r$. Define $\phi : \{\pm 1\}^m \to \reals$ by $\phi(\epsilon_1, \ldots, \epsilon_m) = \frac 1m \sup_{h \in \Hcal_j} \left| \sum_{i=1}^m \epsilon_i h(\bx_i) \right|$. Note that for all $\epsilon_1, \ldots, \epsilon_m$, $1 \leq i \leq m$,
	$$
	\left| \phi(\ep_1, \ldots, \ep_{i-1}, 1, \ep_{i+1}, \ldots, \ep_m) - \phi(\ep_1, \ldots, \ep_{i-1}, -1, \ep_{i+1}, \ldots, \ep_m) \right| \leq 2A/m,
	$$
	so $\phi$ satisfies a bounded differences assumption with variance factor $A^2/m$. By McDiarmid's inequality (Theorem 6.2, \cite{boucheron2013concentration}) it follows that for $t > 0$,
	$$
	\p_{\ep_1, \ldots, \ep_m} \left[ \frac 1m \sup_{h \in \Hcal_j} \left| \sum_{i=1}^m \ep_i h(\bx_i) \right| - \hat{\Rcal}(\Hcal_j) > t\right] \leq \exp(-t^2m/(2A^2)).
	$$
	(The same inequality, up to a constant factor, also follows directly from Theorem 4.7 in \cite{ledoux2013probability}.) A union bound then shows that for $t > 0$,
	\begin{eqnarray}
	&&\p_{\ep_1, \ldots, \ep_m} \left[ \frac 1m \sup_{h \in \bigcup_j \Hcal_j} \left| \sum_{i=1}^m \ep_i h(\bx_i) \right| - \max_{1 \leq j \leq r} \hat{\Rcal}(\Hcal_j) > t + \sqrt{\frac{ 2 A^2 \log r}{m}} \right] \nonumber\\
	&\leq& r \cdot \exp\left( -\left(t + \sqrt{\frac{2 A^2 \log r}{m}}\right)^2\cdot \frac{m}{2A^2}\right)\nonumber\\
	& \leq & r \cdot \exp(-\log r) \cdot \exp(-t^2m/(2A^2)) \nonumber\\
	& = & \exp(-t^2m/(2A^2))\nonumber.
	\end{eqnarray}
	Therefore,
	\begin{eqnarray}
	\hat{\Rcal}\left( \bigcup_{j=1}^r \Hcal_j \right) & = & \E_{\ep_1, \ldots, \ep_m} \left[ \frac 1m \sup_{h \in \bigcup_{j=1}^r \Hcal_j}\left| \sum_{i=1}^m \ep_i h(\bx_i) \right| \right]\nonumber\\
	& \leq & \max_{1 \leq j \leq r} \hat{\Rcal}(\Hcal_j) + \sqrt{\frac{2A^2 \log r}{m}} + \int_0^\infty \exp(-t^2m/(2A^2)) dt\nonumber\\
	&=& \max_{1 \leq j \leq r} \hat{\Rcal}(\Hcal_j) + \sqrt{\frac{2A^2 \log r}{m}} + \frac 12 \cdot \sqrt{\frac{2A^2\pi}{m}}\nonumber\\
	& \leq &  \max_{1 \leq j \leq r} \hat{\Rcal}(\Hcal_j) +2\sqrt 2 A \cdot \frac{\sqrt{\log r}}{\sqrt m} \nonumber.
	\end{eqnarray}

	\end{proof}
	
	We next continue with the proof of \thmref{thm:depthreduc}. It is enough to prove the bound for 
	any fixed $r\in\{1,\ldots,d\}$, and then take the infimum over any such 
	$r$. 
	
	Given $\Hcal$ and $r$, construct a new hypothesis class 
	$\tilde{\Hcal}$ by replacing each 
	network $h\in\Hcal$ by the network $\tilde{h}$ as defined in 
	\thmref{thm:rankonereplace} (namely, where the parameter matrix in the 
	$r'$-th 
	layer is replaced by a rank-1 matrix, for some $r' \in \{1, \ldots, r\}$). We will use the notation 
	$\tilde{h}_h$ 
	to clarify the dependence of $\tilde{h}$ on $h$. According to that theorem, 
	as 
	well as the definition of Rademacher complexity, we have
	\begin{align}
	\hat{\Rcal}_m(\ell\circ\Hcal)&=
	\E_{\boldsymbol{\epsilon}}\left[\sup_{h\in\Hcal}\frac{1}{m}\sum_{i=1}^{m}
	\epsilon_i \ell(h(\bx_i))\right]\notag\\
	&=
	\E_{\boldsymbol{\epsilon}}\left[\sup_{h\in\Hcal}\left(\frac{1}{m}\sum_{i=1}^{m}
	\epsilon_i 
	\ell(\tilde{h}_h(\bx_i))+\frac{1}{m}\sum_{i=1}^{m}\epsilon_i\left(\ell(h(\bx_i))-
	\ell(\tilde{h}_h(\bx_i))\right)\right)\right]\notag\\
	&\leq
	\E_{\boldsymbol{\epsilon}}\left[\sup_{h\in\Hcal}\left(\frac{1}{m}\sum_{i=1}^{m}
	\epsilon_i 
	\ell(\tilde{h}_h(\bx_i))+\sup_{\bx\in\Xcal}|\ell(h(\bx))-\ell(\tilde{h}_h(\bx))|\right)
	\right]\notag\\
	&\leq
	\E_{\boldsymbol{\epsilon}}\left[\sup_{\tilde{h}\in\tilde{\Hcal}}\left(\frac{1}{m}\sum_{i=1}^{m}
	\epsilon_i 
	\ell(\tilde{h}(\bx_i))+\frac{1}{\gamma}\cdot\sup_{\bx\in\Xcal}|h(\bx)-\tilde{h}_h(\bx)|\right)
	\right]\notag\\
	&\leq
	\hat{\Rcal}_m(\ell\circ\tilde{\Hcal})
	+\frac{B\left(\prod_{j=1}^{d}M(j)\right)}{\gamma}
	\left(\frac{2p\log\left(\frac{1}{\Gamma}
		\prod_{j=1}^{d}M_p(j)\right)}{r}\right)^{1/p}~.
	\label{eq:depthreduc1}
	\end{align}
	
	We now reach the crucial observation which lies at the heart of the proof. 
	Consider any network $N_{\tilde{W}_1^d}$ in $\tilde{\Hcal}$, and let 
	$s\bu\bv^\top$ be the singular value decomposition of its (rank-1) parameter matrix 
	$\tilde{W}_{r'}$ in layer $r'$ (where $r' \in \{1, \ldots, r\}$ as in \thmref{thm:rankonereplace}; $s,\bu,\bv$ are leading singular 
	value 
	and vectors of ${W}_{r'}$, by construction). By definition, we have that the 
	composition of $N_{\tilde{W}_1^d}$ with any loss $\ell_j$ equals 
	\[
	\bx\mapsto \ell_j(W_d\sigma_{d-1}(W_{d-1}\sigma_{d-2}(\ldots 
	\sigma_{r'}(s\bu\bv^\top\sigma_{r'-1}(\ldots\sigma_1(W_1\bx)))).
	\]
	This function is equivalent to the composition of the function
	\[
	\bx\mapsto 
	s\bv^\top 
	\sigma_{r'-1}(\ldots\sigma_1(W_1\bx))
	\]
	with the \emph{univariate} function 
	\[
	x\mapsto 
	\ell_j(W_d\sigma_{d-1}(W_{d-1}\sigma_{d-2}(\ldots\sigma_{r'}(\bu x))))~.
	\]
	Note that 
	since $\norm{s\bv^\top}=s=\norm{\tilde{W}_r}$ and 
	$\norm{s\bv^\top}_p=s\leq\norm{W_r}_p$,\footnote{In this equation, by an abuse of notation, $\norm{s\bv^\top}_p$ refers to the $p$-Schatten norm of the {\it matrix} $s\bv$, equivalently, the $p$-Schatten norm of the matrix $s\bv\bv^\top$.} the former function 
	is contained in 
	$\Hcal_{r'}$ as defined in the theorem; whereas the latter function has 
	Lipschitz 
	constant at most $\frac{1}{\gamma}\prod_{j=r'+1}^{d}\norm{W_j}\leq 
	\frac{1}{\gamma}\prod_{j=r'+1}^{d}M(j)$, and 
	maps the input $0$ to the same fixed output $a$ for 
	all 
	$j$. Therefore, we obtain that $\ell\circ\tilde{\Hcal}$ is contained in the following union:
	\[
	\bigcup_{r' \in \{1, \ldots, r\}} \Fcal_{\frac{1}{\gamma}\prod_{j=r'+1}^{d}M(j),a}\circ  \Hcal_{r'}.
	\]
	Notice that for each $r' \in \{1, \ldots, r\}$, 
	all functions of $\Hcal_{r'}$ have output is bounded in $\pm 
	B\prod_{j=1}^{r'}M(j)$). Recall also that
	$\Fcal_{\frac{1}{\gamma}\prod_{j=r'+1}^{d}M(j),a}$ is the class consisting of real-valued
	$\prod_{j=r'+1}^{d}M(j)$-Lipschitz functions $f$ such that $f(0) = a$. As a 
	result, we can apply 
	\thmref{thm:radlipschitz} and obtain that for $r' \in \{1, \ldots, r\}$,
	\begin{align*}
	&\hat{\Rcal}_m\left(\Fcal_{\frac{1}{\gamma}\prod_{j=r'+1}^{d}M(j),a}\circ\Hcal_{r'}\right)\\
	&\leq
	\frac{c}{\gamma}\left(\prod_{j=r'+1}^{d}M(j)\right)\left(\frac{B}{\sqrt{m}}\prod_{j=1}^{r'}M(j)
	+\log^{3/2}(m)\cdot\hat{\Rcal}_m(\Hcal_{r'})\right)\\
	&=
	\frac{c}{\gamma}\left(\prod_{j=1}^{d}M(j)\right)\left(\frac{B}{\sqrt{m}}
	+\frac{\log^{3/2}(m)\cdot\hat{\Rcal}_m(\Hcal_{r'})}
	{\prod_{j=1}^{r'}M(j)}\right)~.
	\end{align*}
	Notice that all functions in $\bigcup_{r' \in \{1, \ldots, r\}}\Fcal_{\frac{1}{\gamma}\prod_{j=r'+1}^{d}M(j),a}\circ\Hcal_{r'}$ have output bounded in $\pm A$, where 
	\[ 
	A := \frac{B}{\gamma}\prod_{j=1}^d M(j) + |a|.
	\]
	By Lemma \ref{lem:radunion}, for an appropriate constant $c'$,
	\[
	\hat{\Rcal}(\ell \circ \tilde{\Hcal}) \leq \frac{c}{\gamma}\left(\prod_{j=1}^{d}M(j)\right)\left(\frac{B}{\sqrt{m}}
	+\log^{3/2}(m)\cdot \max_{r' \in \{1, \ldots, r\}}\frac{\hat{\Rcal}_m(\Hcal_{r'})}
	{\prod_{j=1}^{r'}M(j)}\right) +c' \left(\frac{B}{\gamma} \prod_{j=1}^d M(j) + |a|\right) \cdot \sqrt{\frac{\log r}{m}}.
	\]
	
	Plugging this back into \eqref{eq:depthreduc1} and simplifying a bit (also 
	noting that $p^{1/p}$ can be upper bounded by a universal constant), we get 
	that $\hat{\Rcal}_m(\ell\circ\Hcal)$ is upper bounded by
	\[
	\frac{c''B\prod_{j=1}^{d}M(j)}{\gamma}
	\left(\log^{3/2}(m)\cdot \max_{r' \in \{1, \ldots, r\}}\frac{\hat{\Rcal}_m(\Hcal_{r'})}{B\prod_{j=1}^{r'}M(j)}+
	\left(\frac{\log\left(\frac{1}{\Gamma}\prod_{j=1}^{d}M_p(j)\right)}{r}\right)^{1/p}+
	\frac{1+\sqrt{\log r}}{\sqrt{m}}\right) + |a| \sqrt{\frac{\log r}{m}}.
	\]
	for an appropriate constant $c''$. As mentioned at the beginning of the 
	proof, 
	this upper bound holds for any 
	fixed $r\in\{1,\ldots,d\}$, from which the result follows using the assumption that $|a| \leq \frac{B\prod_{j=1}^d M(j)}{\gamma}$.

	\subsection{Proof of \lemref{lem:ralphabeta}}
	We will show that for any $\alpha,\beta,b,c,n$ as stated in the lemma, 
	there 
	always exists a choice of $r\in\{1,\ldots,d\}$ such that
	\[
	\min\left\{\frac{cr^\alpha}{n}+\frac{b}{r^\beta}~,~\frac{d^\alpha}{n}\right\}~\leq~
	3\cdot\frac{b^{\frac{\alpha}{\alpha+\beta}}}{(n/c)^{\frac{\beta}{\alpha+\beta}}}~.
	\]
	Since the left hand side is also trivially at most $\frac{d^\alpha}{n}$, 
	the result follows. We prove this inequality by a case analysis:
	\begin{itemize}
		\item If $(bn/c)^{\frac{1}{\alpha+\beta}}\in [1,d]$, pick 
		$r=\left\lfloor
		(bn/c)^{\frac{1}{\alpha+\beta}}\right\rfloor\in \{1,2,\ldots,d\}$, in 
		which case
		\[
		\frac{cr^\alpha}{n}+\frac{b}{r^\beta}~\leq~
		\frac{c\left((bn/c)^{\frac{1}{\alpha+\beta}}\right)^\alpha}{n}+\frac{
			b}{\left(\frac{1}{2}(bn/c)^{\frac{1}{\alpha+\beta}}\right)^\beta}
		~=~
		\frac{(b)^{\frac{\alpha}{\alpha+\beta}}}{(n/c)^{\frac{\beta}{\alpha+\beta}}}
		+2^\beta\frac{b^{\frac{\alpha}{\alpha+\beta}}}{(n/c)
			^{\frac{\beta}{\alpha+\beta}}}
		~\leq~3\frac{b^{\frac{\alpha}{\alpha+\beta}}}{(n/c)^{\frac{\beta}{\alpha+\beta}}}~.
		\]
		\item If $(bn/c)^{\frac{1}{\alpha+\beta}}>d$, it follows that
		\[
		\min\left\{\frac{cr^\alpha}{n}+\frac{b}{r^\beta}~,~\frac{d^\alpha}{n}\right\}
		~\leq~\frac{d^\alpha}{n}~<~c\frac{(bn/c)^{\frac{\alpha}{\alpha+\beta}}}{n}
		~=~\frac{b^{\frac{\alpha}{\alpha+\beta}}}{(n/c)^{\frac{\beta}{\alpha+\beta}}}~.
		\]
	\end{itemize}

	%\subsection{Proof of \lemref{lem:ralphabeta}}
	%\begin{proof}
	%First suppose $\left( \frac{nb\beta}{\alpha} \right)^{1/(\alpha + \beta)} 
	%\leq 
	%d$. Then we choose $r = \left( \frac{nb\beta}{\alpha} %\right)^{1/(\alpha 
	%%+ 
	%\beta)}$. (This is the only critical point of the %continuously 
	%%differentiable 
	%and non-negative function $r \rightarrow %\frac{r^\alpha}{n} + 
	%\frac{b}{r^\beta}$, for $r > 0$; since the function goes %to $\infty$ as 
	%%$r 
	%\rightarrow^+ 0, r \rightarrow \infty$, it must be the %case that this 
	%%critical 
	%point is a global minimum.)
	%
	%Then for this value of $r$,
	%\[
	%\frac{r^\alpha}{n} + \frac{b}{r^\beta} = \frac{b^{\frac{\alpha}{\alpha + 
	%\beta}}}{n^{\frac{\beta}{\alpha + \beta}}} \cdot \left( 
	%(\frac{\beta}{\alpha})^{\alpha/(\alpha + \beta)} + 
	%(\frac{\alpha}{\beta})^{\beta/(\alpha + \beta)} \right) \leq 2 
	%\frac{b^{\frac{\alpha}{\alpha + \beta}}}{n^{\frac{\beta}{\alpha + \beta}}},
	%\]
	%since $x^{1/(x+1)} + (1/x)^{1/(1/x + 1)} \leq 2$ for all $x > 0$.
	%
	%Now suppose $\left( \frac{nb\beta}{\alpha} \right)^{1/(\alpha + \beta)} > 
	%d$. 
	%Then
	%\[
	%\frac{d^\alpha}{n} \leq (\beta/\alpha)^{\alpha/(\alpha + \beta)} \cdot 
	%\frac{b^{\alpha/(\alpha + \beta)}}{n^{\beta/(\alpha + \beta)}} < 2 
	%\frac{b^{\alpha/(\alpha + \beta)}}{n^{\beta/(\alpha + \beta)}},
	%\]
	%so the result follows in this case as well.
	%\end{proof}

	\subsection{Proof of Corollary \ref{cor:sqrtmrate}}
	
	A direct application of \thmref{thm:sqrtl}, as well as the fact that the 
	loss $\ell$ is $1/\gamma$ Lipschitz, implies that
	\begin{equation}\label{eq:sqrtmrate1}
	\hat{\Rcal}_m(\ell\circ\Hcal)~\leq~ \Ocal\left(
	\frac{B \prod_{j=1}^d 
		M_F(j)}{\gamma}\sqrt{\frac{d}{m}}
	\right).
	\end{equation}
	Next, we plug \eqref{eq:radd23} into \thmref{thm:depthreduc}. 
	We use $p=2$, let each $\Wcal_j$ be the space of all matrices, and 
	choose $M(j)=M_F(j)$ for all $j$, noting that 
	$\frac{\norm{W_j}}{M(j)}\leq \frac{\norm{W_j}_F}{M_F(j)}\leq 1$. Since $\sqrt r \geq \sqrt{r'}$ for all $r' \leq r$, it follows that % Moreover,
	% since the activations $\sigma_j$ are assumed to be positive homogeneous, we may assume without loss of generality that $M_F(1), \ldots, M_F(d) \geq 1$ since we may multiply each of $M_F(1), \
	% \ldots, M_F(d)$ by a sufficient value so that all are at least 1, and then divide $B$ by the product of those values, without changing the Rademacher complexity. It then follows from \eqref{eq:radd23}, \thmref{thm:sqrtl}, and the assumption that $1 \leq B \prod_{j=1}^d M_F(j)$ that
	\begin{equation}\label{eq:sqrtmrate2}
	\hat{\Rcal}_m(\ell\circ\Hcal)~\leq~\Ocal\left(\frac{B \prod_{j=1}^d 
		M_F(j)}{\gamma} \left( \min_{r \in \{1 ,\ldots, d\}} \left\{ \frac{ 
		\log^{3/2}(m) \sqrt r 
	}{\sqrt{m} } + \sqrt{\frac{\log\left(\frac 1\Gamma \prod_{j=1}^d 
			M_F(j)\right)}{r}} \right\}\right) \right)~.
	\end{equation}
	We remark that the $\Ocal \left(\frac{1+\sqrt{\log r}}{\sqrt m}\right)$ term from \thmref{thm:depthreduc} is absorbed into the $\Ocal\left( \frac{\log^{3/2}( m) \sqrt{r}}{\sqrt m} \right)$ term in the minimization over $r$ in \eqref{eq:sqrtmrate2}. 
	Upper bounding $\hat{R}_m(\ell\circ\Hcal)$ by the minimum of 
	\eqref{eq:sqrtmrate1} and \eqref{eq:sqrtmrate2}, and using 
	\lemref{lem:ralphabeta} with $\alpha=\beta=\frac{1}{2}$, 
	$b=\sqrt{\bar{\log}\left(\frac 1\Gamma \prod_{j=1}^d M_F(j)\right)}$, 
	$c=\bar{\log}^{3/2}(m)$, and $n=\sqrt{m}$, the result follows.
	
	\subsection{Proof of \corollaryref{cor:sqrtmrate2}}
	Consider the class
	\[
	\Hcal_r =\left\{N_{W_1^r}~:~\begin{matrix}N_{W_1^r}\text{ maps to 
		$\reals$}\\
	\forall j\in\{1\ldots 
	r-1\},~~\frac{\norm{W_j^T}_{2,1}}{\norm{W_j}}\leq L\\
	\forall j\in\{1\ldots 
	r\},~~\max\left\{\frac{\norm{W_j}}{M(j)},\frac{\norm{W_j}_p}{M_p(j
		)}\right\}\leq 1
	\end{matrix}\right\}~.
	\]
	
	Since $N_{W_1^r}$ in the definition of $\Hcal_r$ maps to $\reals$, $W_r$ is 
	a vector, meaning that $\| W_r^T\|_{2,1} = \| W_r \|_2 = \| W_r\|$. 
	Therefore, we can use \thmref{thm:bartlett} to bound $\hat 
	\Rcal_m(\Hcal_r)$; in particular, for $r \geq 1$,
	\[
	\hat \Rcal_m (\Hcal_r) \leq \Ocal \left( \frac{B \log(h) \log(m) L 
		r^{3/2}\prod_{j=1}^r M(j)}{\sqrt m} \right).
	\]
	
	Since $r^{3/2} \geq (r')^{3/2}$ for all $r' \leq r$, it therefore follows from \thmref{thm:depthreduc} that 
	$\hat{\Rcal}_m(\ell\circ\Hcal)$ is at most 
	\[
	\Ocal\left(\frac{BL \log(h) \log(m) \prod_{j=1}^{d}M(j)}{\gamma}
	\cdot
	\min_{r\in\{1,\ldots,d\}}\left\{\frac{cr^{3/2}}{n} +
	\frac{b}{r^{1/p}}\right\}\right)~,
	\]
	where $c=\bar{\log}^{3/2}(m)$, $n=\sqrt{m}$, and
	$b = \bar{\log}\left( \frac{1}{\Gamma} \prod_{j=1}^d M_p(j) \right)^{1/p}$ 
	(note that we use here $\bar{\log}$ instead of $\log$ so the conditions of 
	\lemref{lem:ralphabeta}, to be used shortly, will be satisfied). 
	As in \corollaryref{cor:sqrtmrate}, the $\Ocal\left(\frac{1 + \sqrt{\log 
	r}}{\sqrt m}\right)$ term from \thmref{thm:depthreduc} is absorbed into the 
	$\Ocal\left( \frac{cr^{3/2}}{n} \right)$ term in the minimization over $r$ 
	in the above equation.
	
	On the other hand, a direct application of \thmref{thm:bartlett} also 
	implies 
	that $\hat{\Rcal}_m(\ell\circ\Hcal)$ is at most 
	\[
	\Ocal \left( \frac{B \prod_{j=1}^d M(j)}{\gamma} \cdot \frac{\log(h) 
	\log(m) 
		Ld^{3/2}}{\sqrt m}  \right).
	\]
	Combining the above two bounds, and applying Lemma \ref{lem:ralphabeta}, we 
	get 
	that $\hat{\Rcal}_m(\ell\circ\Hcal)$ is at most
	\[
	\Ocal \left( \frac{BL \log(h) \log(m) \prod_{j=1}^d M(j)}{\gamma}
	\cdot \min \left\{ \frac{\bar\log \left( \frac{1}{\Gamma} \prod_{j=1}^d M_p(j) 
		\right)^{\frac{1}{\frac{2}{3}+p}} \left( 
		\bar\log^{3/2}(m)\right)^{\frac{1}{1+\frac{3}{2}p}}}{m^{\frac{1}{2+3p}}},\frac{d^{3/2}}{\sqrt
		m} \right\} \right)~.
	\]
	
	%\begin{eqnarray}
	%\hat \Rcal_m(\Hcal) &\leq&\notag \Ocal \left( \frac{BL \log(h) \log(m) 
	%\prod_{j=1}^r M(j)}{\gamma}\right.\notag\\
	%&&\left.\cdot \min \left\{ \frac{\log \left( \frac{1}{\Gamma} 
	%\prod_{j=1}^d 
	%M_p(j) \right)^{3/8} \left(L 
	%\log^{3/2}(m)\right)^{1/4}}{m^{1/8}},\frac{d^{3/2}}{\sqrt m} \right\} 
	%\right)\notag.
	%\end{eqnarray}

	\subsection{Proof of \thmref{thm:lowerbound}}
	
	By definition of the Rademacher complexity, it is enough to lower bound the 
	complexity of $\ell\circ \Hcal'$ where $\Hcal'$ is some subset of $\Hcal$. 
	In particular, consider the class $\Hcal'$ of neural networks over 
	$\reals^h$ of the form
	\[
	\bx\mapsto M_p(d)\cdot M_p(d-1)\cdots M_p(2)\cdot \sigma(W\bx),
	\]
	where $W=\text{diag}(\bw)$ is an $h\times h$ diagonal matrix satisfying 
	$\norm{\bw}_p\leq M_p(1)$ (here, $\norm{\cdot}_p$ refers to the vector 
	$p$-norm), and $\sigma(\bz)=\max_j z_j$. Furthermore, suppose that 
	$\ell(z)=\frac{1}{\gamma}z$. Finally, we will choose 
	$\bx_1,\ldots,\bx_m$ in $\reals^h$ as $\bx_i=B\be_{(i~\text{mod}~h)}$ for 
	all $i$, where $\be_t$ is the $t$-th standard basis vector.. Letting 
	$A_k=\{i\in \{1,\ldots,m\}:i~\text{mod}~h=k\}$, it holds that
	\begin{align*}
	\hat{\Rcal}_m(\ell\circ\Hcal')~&=~
	\E_{\boldsymbol{\epsilon}}\sup_{\bw:\norm{\bw}_p\leq M_p(1)}
	\frac{1}{m}\sum_{i=1}^{m}\epsilon_i \ell\left( 
	\prod_{j=2}^{d}M_p(j)\sigma(\text{diag}(\bw)\bx_i)\right)\\
	&=~ \frac{B\prod_{j=2}^{d}M_p(j)}{\gamma m}\cdot 
	\E_{\boldsymbol{\epsilon}}\sup_{\bw:\norm{\bw}_p\leq M_p(1)} 
	\sum_{k=1}^{h}\max\{0,w_k\}\cdot\sum_{i\in A_k}\epsilon_i\\
	&=~ \frac{B\prod_{j=1}^{d}M_p(j)}{\gamma m}\cdot 
	\E_{\boldsymbol{\epsilon}}\sup_{\bw:\norm{\bw}_p\leq 1} 
	\sum_{k=1}^{h}\max\{0,w_k\}\cdot\sum_{i\in A_k}\epsilon_i.
	\end{align*}
	In particular, by choosing $w_k = h^{-1/p}\cdot\text{sign}\left(\sum_{i\in 
		A_k}\epsilon_i\right)$ for all $k$, we can lower bound the above by
	\begin{align*}
	\frac{B\prod_{j=1}^{d}M_p(j)}{\gamma m\cdot h^{1/p}}\cdot 
	\E_{\boldsymbol{\epsilon}}\left[ \sum_{k=1}^{h}\max\left\{0,\sum_{i\in 
		A_k}\epsilon_i\right\}\right]~=~
	\Omega\left(
	\frac{B\prod_{j=1}^{d}M_p(j)}{\gamma m\cdot 
		h^{1/p}}\cdot\sum_{k=1}^{h}\sqrt{|A_k|}\right)~,
	\end{align*}
	and since $|A_k|\geq \lfloor m/h\rfloor$ by its definition, we get a lower 
	bound of
	\begin{equation}\label{eq:lowbound1}
	\Omega\left(
	\frac{B\prod_{j=1}^{d}M_p(j)}{\gamma m\cdot h^{1/p}}\cdot h 
	\sqrt{\frac{m}{h}}\right)
	~=~
	\Omega\left(
	\frac{B\prod_{j=1}^{d}M_p(j)\cdot 
		h^{\frac{1}{2}-\frac{1}{p}}}{\gamma\sqrt{m}}\right)~.
	\end{equation}
	An alternative bound (which is better when $p<2$) can be obtained by 
	considering the class $\Hcal'$ of real-valued neural networks over 
	$\reals$, of the form
	\[
	x\mapsto M_p(d)\cdot M_p(d-1)\cdots M_p(2)\cdot wx,
	\]
	where $|w|\leq M_p(1)$. Furthermore, supposing again that $\ell(z)=\frac{1}{\gamma}z$, and that $\bx_i=B$ for all $i$, it holds that
	\begin{align*}
	\hat{\Rcal}_m(\ell\circ \Hcal') ~&=~ \E_{\boldsymbol{\epsilon}} 
	\sup_{w:|w|\leq M_p(1)} \frac{1}{\gamma m}\sum_{i=1}^{m}\epsilon_i 
	\prod_{j=2}^{d}M_p(j)Bw\\
	&=~ \frac{B\prod_{j=1}^{d}M_p(j)}{\gamma m}\cdot 
	\E_{\boldsymbol{\epsilon}}\left|\sum_{i=1}^{m}\epsilon_i\right|
	~=~\Omega\left(\frac{B\prod_{j=1}^{d}M_p(j)}{\sqrt{m}}
	\right)~.
	\end{align*}
	Taking the best of this lower bound and the lower bound in 
	\eqref{eq:lowbound1}, the result follows.
	
	\section{Acknowledgements}
	We thank Ziwei Ji and Matus Telgarsky for pointing that an additional union 
	bound is required in the proof of \thmref{thm:depthreduc}, resulting in an 
	extra logarithmic factor, as well as Pritish Kamath for pointing out a bug in the proof of Theorem \ref{thm:lowerbound}, resulting in a slight change to the construction. Noah Golowich is grateful for research funding from Harvard University, including the Herchel Smith and PRISE fellowships.
	
	%\subsection{Proof of Corollary \ref{cor:lip_margin}}
	%
	%Theorem \ref{thm:kol-pan} applied to $\Gcal(1)\cap \Hcal$ yields: with 
	%probability at least $1-2\exp\{-2t^2\}$, 
	%    \begin{align*}
	%        &\forall L>0, \forall f\in (1/L) \Gcal(L)\cap\Hcal, 
	%\gamma\in(0,B],\\ 
	%        &~~~~~~~ P(f\leq 0) \leq P_m \phi(f/\gamma) + 
	%\frac{8L_\phi}{\gamma} 
	%\hat\Rcal_m(\Gcal(1)\cap \Hcal) + \left(\frac{\log\log 
	%2B\gamma^{-1}}{m}\right)^{1/2} + \frac{t}{\sqrt{m}}.
	%    \end{align*}
	%    For any neural network $g$ with finite Lipschitz constant (that is, 
	%$g\in 
	%\cup_{L>0} \Gcal(L)\cap \Hcal$), the above inequality applied to $g/L_g$ 
	%gives
	%    \begin{align*}
	%        &\forall g\in  \cup\Gcal(L)\cap\Hcal, \gamma\in(0,B], \\
	%        &~~~~~~~ P(g\leq 0) \leq P_m \phi(g/(L_g\gamma)) + 
	%\frac{8L_\phi}{\gamma} \hat\Rcal_m(\Gcal(1)\cap \Hcal) + 
	%\left(\frac{\log\log 
	%2B\gamma^{-1}}{m}\right)^{1/2} + \frac{t}{\sqrt{m}}
	%    \end{align*}
	%    and setting $\gamma'=L_g \gamma$,
	%    \begin{align*}
	%        &\forall g\in  \cup\Gcal(L)\cap\Hcal, \gamma'\in(0,B\cdot L_g],\\ 
	%        &~~~~~~~ P(g\leq 0) \leq P_m \phi(g/\gamma') + \frac{8L_\phi 
	%L_g}{\gamma'} \hat\Rcal_m(\Gcal(1)\cap \Hcal) + \left(\frac{\log\log 
	%2B(\gamma'/L_g)^{-1}}{m}\right)^{1/2} + \frac{t}{\sqrt{m}}.
	%    \end{align*}
	%    We take $\phi$ to be the $1$-Lipschitz ramp function. Then 
	%$P_m\phi(g/\gamma') \leq P_m\ind{g\leq \gamma'}$, the proportion of data 
	%within 
	%$\gamma'$ margin, and we get the statement of the corollary.

	\bibliographystyle{plainnat}
	\bibliography{bib}

\end{document}